\documentclass[journal]{IEEEtran} %compsoc

\usepackage{times}
\usepackage{graphicx} % more modern
\usepackage{subcaption}
\usepackage[square,sort,comma,numbers]{natbib}
 
\vspace{-10mm}

\usepackage{amsmath}
\usepackage{amsthm}
\usepackage{algorithm}
\usepackage{algorithmic}

\usepackage[english]{babel}
\usepackage{subcaption}
\newtheorem{mydef}{Definition}
\usepackage{mathtools}
\newtheorem{theorem}{Theorem}
\theoremstyle{definition}

\newtheorem{corollary}{Corollary}[theorem]
\newtheorem{lemma}[theorem]{Lemma}
\usepackage{kantlipsum} %<- For dummy text
\newtheorem{asu}{Assumption}
\usepackage{amsfonts}
\usepackage{bbm}
\usepackage{mathptmx}
\newcommand{\subparagraph}{}
\usepackage{titlesec}
\titlespacing{\section}{0pt}{5pt}{-\parskip}
\titlespacing{\subsection}{0pt}{5pt}{-\parskip}
\DeclareCaptionLabelSeparator{periodspace}{.\quad}
\ifCLASSINFOpdf
\else
\fi

\begin{document}
\title{Dynamic Differential Privacy for Distributed Machine Learning over Networks}

\author{Tao Zhang, Student Member, and Quanyan Zhu, Member \thanks{T. Zhang and Q. Zhu are with the Department of Electrical and Computer Engineering, Tandon School of Engineering, New York University, Brooklyn, NY; Email:\{tz636,qz494\}@nyu.edu\}}}

% make the title area
\maketitle

% As a general rule, do not put math, special symbols or citations
% in the abstract

\IEEEpeerreviewmaketitle

\begin{abstract}
Privacy-preserving distributed machine learning becomes increasingly important due to the recent rapid growth of data. This paper focuses on a class of regularized empirical risk minimization (ERM) machine learning problems, and develops two methods to provide differential privacy to distributed learning algorithms over a network. 
%algorithms We use the definition of differential privacy, developed by Dwork et al. privacy to capture the notion of privacy of our algorithm. We provide two methods. 
%
We first decentralize the learning algorithm using the alternating direction method of multipliers (ADMM), and propose the methods of \textit{dual variable perturbation} and  \textit{primal variable perturbation} to provide dynamic differential privacy. %which perturbs the dual variable before next intermediate minimization of augmented Lagrange function over the classifier in every ADMM iteration. 
%
%Another approach we apply the output perturbation to the primal variable before releasing it to neighboring nodes. We call the second method \textit{primal variable perturbation}.
%
The two mechanisms lead to algorithms that can provide privacy guarantees under mild conditions of the convexity and differentiability of the loss function and the regularizer. We study the performance of the algorithms, and show that the dual variable perturbation outperforms its primal counterpart.
% our algorithms is proved to provide differential privacy through the entire learning process. We also provide theoretical results for the accuracy of the algorithm, and prove that both algorithms converges in distribution. The theoretical results show that the dual variable perturbation outperforms the primal case.
%
To design an optimal privacy mechanisms, we analyze the fundamental tradeoff between privacy and accuracy, and provide guidelines to choose privacy parameters. Numerical experiments using customer information database are performed to corroborate the results on privacy and utility tradeoffs and design. %  is examined in the numerical experiment. Our experiment shows that both algorithms performs similar in managing the privacy-accuracy tradeoff, and primal variable perturbaiton is slightly better than the dual case.
\end{abstract}

\section{Introduction}

%The data deluge 
\textit{Distributed machine learning} is a promising way to manage deluge of data that has been witnessed recently. With the training data of size
ranging from 1$TB$ to 1$PB$ \cite{li2014scaling}, a centralized machine learning approach that collects and processes the data can lead to significant computational complexity and communications overhead. Therefore,  a decentralized approach to machine learning is imperative to provide the scalability of the data processing and improve the quality of decision-making, while reducing the computational cost.

One suitable approach to decentralize a centralized machine learning problem is \textit{alternating direction method of multiplier} (ADMM). It enables distributed training over a network of collaborative nodes who exchange their results with the neighbors. However, the communications between two neighboring nodes create serious privacy concerns for nodes who process sensitive data including social network data, web search histories, financial information, and medical records. An adversary can observe the outcome of the learning and acquire sensitive information of the training data of individual nodes. The adversary can be either a member of the learning network who observes its neighbors or an outsider who observes the entire network. A privacy-preserving mechanism needs to automatically build into the distributed machine learning scheme to protect the internal and external adversaries throughout the entire dynamic learning process. Differential privacy is a suitable concept that provides a strong guarantee that the removal or addition of a single database item does not allow an adversary to distinguish (substantially) an individual data point \cite{dwork2006calibrating}. 

In this work, we focus on a class of distributed ADMM-based \textit{empirical risk minimization} (ERM) problems, and  develop randomized algorithms that can provide differential privacy \cite{dwork2006calibrating, nissim2007smooth} while keeping the learning procedure accurate. We extend the privacy concepts to dynamic differential privacy to capture the nature of distributed machine learning over networks, and propose two privacy-preserving schemes of the regularized ERM-based optimization. The first method is \textit{dual variable perturbation} (DVP), in which we perturb the dual variable of each node at every ADMM iteration. The second is the \textit{primal variable perturbation} (PVP) which leverages the \textit{output perturbation} technique developed by Dwork et al. \cite{dwork2006calibrating} by adding noise to the update process of primal variable of each node of the ADMM-based distributed algorithm before sharing it to neighboring nodes.

 We investigate the performance of the algorithms, and show that the DVP outperforms PVP. We characterize the fundamental tradeoffs between privacy and accuracy by formulating an optimization problem and use numerical experiments to demonstrate the optimal design of privacy mechanisms. The main contributions of the paper are summarized as follows: % the tradeoffs 

%Our results are generally applicable to ERM optimization problems with different loss functions, and we use  logistic regression as an example to  conduct numerical experiments. 

%Since there are no conditions for the dataset for the purpose of privacy preservation, the randomness incurs a cost in the performance while guaranteeing the differential privacy. Therefore, managing the tradeoff between privacy and accuracy is critical. Under the assumption that the data points in the dataset are drawn from an unknown but fixed distribution, we prove the accuracy of the distributed learning algorithm in terms of the privacy parameters. Another important issue is the convergence of ADMM. There are many convergences results for ADMM discussed in literature. Based on the accuracy analysis, we also discuss the convergence of our private ADMM method.
 
%The contributions of this paper are shown as follows:
\begin{description}
  \item[(i)] We use ADMM to decentralize regularized ERM algorithms to achieve distributed training of large datasets. Dynamic differential privacy is guaranteed for the distributed algorithm using the DVP, which adds noise to the update of the dual variable.% in which we add randomness to the dual variable before the next update of the primal variable. The differential privacy is guaranteed for every ADMM iteration as well as the final trained output. 
  
  \item[(ii)] We develop PVP method to add noise to the primal variables when they are transmitted to neighboring nodes. This approach guarantees dynamic differential privacy in which privacy is preserved at each update. %In this technique, the randomness rises when every node transmits the primal parameter to the corresponding neighboring nodes.  %It is guaranteed to provide differential privacy for the every intermediate update. For the final update, we apply the dual variable perturbation in order to increase the accuracy.
  
  \item[(iii)] We provide the theoretical performance guarantees of the PVP perturbations of the distributed ERM with $l_2$ regularization. The performance is measured by the number of sample data points required to achieve a certain criteria. Our theoretical results show that DVP is prefered for more difficult learning problems that is non-separable or with small margin.  
  
  \item[(iv)] %We implement our methods by experiments on a dateset of UCI Machine Learning Repositories [16]. We provide a method to select the optimal privacy parameter $\alpha$ by solving an optimization problem given a specific utility function of privacy. 
  We propose a design principle to  select the optimal privacy parameters by solving an optimization problem. Numerical experiments show that the PVP outperforms the DVP at managing the privacy-accuracy tradeoff. %However, theoretical analysis shows that dual variable perturbation has higher probability of accuracy and better sample requirement than does the primal case. Both algorithms are suitable for the both types of attacks we are interested in.
\end{description}

\subsection{Related Work}

There has been a significant amount of literature on the distributed classification learning algorithms. These works have mainly focused on either enhancing the efficiency of the learning model, or on producing a global classifier from multiple distributed local classifier trained at individual nodes. Researchers have focused on making the distributed algorithm suitable to large-scale datasets, e.g., MapReduce has been used to explore the performance improvements \cite{dean2008mapreduce}. In addition, methods such as ADMM methods \cite{forero2010consensus}, voting classification \cite{collins2002discriminative}, and mixing parameters \cite{mcdonald2010distributed} have been used to achieve distributed computation. Our approach to distributed machine learning is based on ADMM, in which the centralized problem acts as a group of coupled distributed convex optimization subproblems with the consensus constraints on the decision parameters over a network. 

In privacy-preserving data mining research, the privacy can be pried through, for example, \textit{composition attacks}, in which the adversary has some prior knowledge. Other works on data perturbation for privacy (e.g., \cite{evfimievski2004privacy},\cite{ kim2003multiplicative}) have focused on additive or multiplicative perturbations of individual samples, which might affect certain relationships among different samples in the database.
A body of existing literature also have studied the differential-private machine learning. For example, Kasiviswanathan et al. have derived a general method for probabilistically approximately correct (PAC, \cite{valiant1984theory}) in \cite{kasiviswanathan2011can}. 
Many works have investigated the tradeoff privacy and accuracy while developing and exploring the theory of differential privacy (examples include \cite{dwork2006calibrating, mcsherry2007mechanism, blum2005practical}). In this work, we extend the notion of differential privacy to a dynamic setting, and define dynamic differential privacy to capture the distributed and iterative nature of the ADMM-based distributed ERM. 

\subsection{Organization of the Paper}
The rest of the paper is organized as follows. Section 2  presents the ADMM approach to decentralize a centralized ERM problem, and describe the privacy concerns associated with the distributed machine learning. In Section 3, we present dual and primal variable perturbation algorithms to provide dynamic differential privacy. The analysis of privacy guarantee for the algorithms is discussed. Section 4 studies the performance of the privacy-preserving algorithms. Section 5 presents numerical experiments to corroborate the results and optimal design principles to tradeoff between privacy and accuracy. Finally, Section 6 presents concluding remarks and future research directions.

\section{Problem Statement}
% Show a centralized problem with a network of nodes
%
% 
Consider a connected network, which contains $P$ nodes described by one undirected graph $G(\mathcal{P}, \mathcal{E})$ with the set of nodes $\mathcal{P} = \{1,2,3,..., P\}$, and a set of edges $\mathcal{E}$ represented by lines denoting the links between connected nodes. A particular node $p \in \mathcal{P}$ only exchanges information between its neighboring node $j\in \mathcal{N}_p$, where $j\in \mathcal{N}_p$ is the set of all neighboring nodes of node $p$, and $N_p = |\mathcal{N}_p|$ is the number of neighboring nodes of node $p$.  
Each node $p$ contains a dataset $D_{p} = \{(x_{ip},y_{ip})\subset X \times Y : i = 0,1,...,B_{p} \},$ which is
of size $B_p$ with data vector $x_{ip} \in X \subseteq \mathbb{R}^d$, and the corresponding label $y_{ip} \in Y := \{-1, 1\}$. The entire network therefore has a set of data  $\hat{D} = \bigcup_{p \in \mathcal{P}} D_{p}.$

The target of the centralized classification algorithm is to find a classifier $f: X\rightarrow Y$ using all available data $\hat{D}$ that enables the entire network to classify any data $x'$ input to a label $y' \in \{-1, 1\}$.
Let $Z_{C_1}(f|\hat{D})$ be the objective function of a regularized empirical risk minimization problem (CR-ERM), defined as follows:
\begin{equation}\label{CRERM}  
Z_{C_1}(f|\hat{D}) := \frac{C^R}{B_p}\sum_{p=1}^{P}\sum_{i=1}^{B_{p}}\mathcal{\hat{L}}(y_{ip},\:\: f^T x_{ip})+\rho R(f),
\end{equation}
where $C^R\leq B_p$ is a regularization parameter, and $\rho>0$ is the parameter that controls the impact of the regularizer.
Suppose that $\hat{D}$ is available to the fusion center node, then we can choose the global classifier $f:X\rightarrow Y$ that minimizes the CR-ERM.

The \textit{loss function} $\mathcal{\hat{L}}(y_{ip},\:\: f^T x_{ip}): \mathbb{R}^d\rightarrow \mathbb{R}$, is used to measure the quality of the classifier trained. In this paper, we focus on the specific loss function
$ \mathcal{\hat{L}}(y_{ip},\:\: f^T x_{ip})=\mathcal{L}(y_{ip} f^T x_{ip})$.
The function $R(f)$ in (\ref{CRERM}) is a regularizer that prevents overfitting. 
In this paper, we have the following assumptions on the loss, regularization functions, and the data.
\begin{asu} 
The loss function $\mathcal{L}$ is strictly convex and doubly differentiable of $f$ with $| \mathcal{L}'| \leq 1$ and $|\mathcal{L}''|\leq c_1$, where $c_1$ is a constant. Both $\mathcal{L}$ and $\mathcal{L}'$ are continuous.
\end{asu}
\begin{asu} 
The regularizer function $R(\cdot)$ is continuous differentiable and 1-strongly convex. Both $R(\cdot)$ and $\nabla R(\cdot)$ are continuous.
\end{asu}
\begin{asu} 
We assume that $\lVert {x_{ip}} \lVert \leq 1$. Since $y_{ip}\in\{-1,1\}$, $|y_{ip}|=1$.
\end{asu}

\subsection{Distributed ERM}

%We aim to solve the CR-ERM problem in a distributed fashion while achieving the same performance as in the centralized case. The decentralized equivalent enables node $p$ to contribute by optimizing only the $p$-dependent terms of the objective function without exchanging any training data to other nodes $p'\neq p$. 

To decentralize CR-ERM, we introduce decision variables $\{f_{p}\}_{p=1}^{P}$, where node $p$ determines its own classifier $f_p$,  and impose consensus constraints $f_1=f_2=...=f_P$ that guarantee global consistency of the classifiers. Let  $ \{w_{jp}\}$ be the auxiliary  variables to decouple $f_p$ of node $p$ from its neighbors $j\in \mathcal{N}_p$. 
Then, the consensus-based reformulation of (\ref{CRERM}) becomes 
\begin{equation}\label{equiCRERM}
\begin{split}
\min_{\{f_{p}\}_{p=1}^{P}}\:\: Z_{C_2}:=\frac{C^R}{B_p}\sum_{p=1}^{P}\sum_{i=1}^{B_{p}}\mathcal{L}(y_{ip} f_{p}^T x_{ip})+\sum_{p=1}^{P}\rho R(f_p).\\ 
\textrm{s.t. \ \ } f_{p} =w_{pj},\: w_{pj}=f_{j}, p = 1,...,P, j \in \mathcal{N}_p
\end{split}
\end{equation}
where  $Z_{C_2}(\{f_p\}_{p\in\mathcal{P}}|\hat{D})$ is the reformulated objective as a function of $\{f_{p}\}_{p=1}^{P}$. According to Lemma $1$ in \cite{forero2010consensus}, if $\{f_{p}\}_{p=1}^{P}$ presents a feasible solution of (\ref{equiCRERM}) and the network is connected, then problems (\ref{CRERM}) and (\ref{equiCRERM}) are equivalent, i.e., $f = f_{p}, p = 1,...,P$, where $f$ is a feasible solution of CR-ERM. Problem (\ref{equiCRERM}) can be solved in a distributed fashion using the alternative direction method of multiplier (ADMM) with each node $p\in\mathcal{P}$ optimizing the following distributed regularized empirical risk minimization problem (DR-ERM):
\begin{equation}\label{equiObjp}
Z_p(f_p|D_p) := \frac{C^R}{B_p}\sum_{i=1}^{B_p}\mathcal{L}(y_{ip}f_p^Tx_{ip}) + \rho R(f_p).
\end{equation}
The augmented Lagrange function associated with the DR-ERM is:
\begin{equation}\label{equiDLag}
\begin{aligned}
&L^D_p(f_{p},w_{pj}, \lambda^{k}_{pj})\\
&=Z_p+\sum_{i\in \mathcal{N}_p}\big(\lambda_{pi}^{a}\big)^T(f_{p}-w_{pi}) +\sum_{i\in \mathcal{N}_p}\big(\lambda_{pi}^{b}\big)^T(w_{pi}-f_{i}) \\
&+\frac{\eta}{2}\sum_{i\in \mathcal{N}_p}( \parallel f_{p}-w_{pi} \parallel^2 + \parallel w_{pi}-f_{i} \parallel^2).
\end{aligned}
\end{equation}

The distributed iterations solving (\ref{equiObjp}) are:
\begin{equation}\label{equifp1}
f_{p}(t+1)= \arg\min_{f_{p}}L_p^D\big(f_{p},w_{pj}(t), \lambda^k_{pj}(t)\big),
\end{equation}
\begin{equation}\label{equiAuxi} 
\begin{aligned}
w_{pj}(t+1) = \arg\min_{w_{pj}}L_p^D\big(f_{p}(t+1),w_{pj}, \lambda^k_{pj}(t)\big),
\end{aligned}
\end{equation}
\begin{equation}\label{equiLamb_1_1}
\begin{aligned}
\lambda^a_{pj}(t+1)=\lambda^a_{pj}(t)&+\eta(f_{p}(t+1)-w_{pj}(t+1)),\\
& p\in\mathcal{P},\:\:j\in\mathcal{N}_p,
\end{aligned}
\end{equation}
\begin{equation}\label{equiLambd_1_2}
\begin{aligned}
\lambda^b_{pj}(t+1)=\lambda^b_{pj}(t)&+\eta(w_{pj}(t+1)-f_{p}(t+1)),\\
& p\in\mathcal{P},\:\:j\in\mathcal{N}_p.
\end{aligned}
\end{equation}

According to Lemma 2 in \cite{forero2010consensus}, iterations (\ref{equifp1}) to (\ref{equiLambd_1_2}) can be further simplified by  
%From (4), the augmented Lagrange function is linear-quadratic in $w_{pi}$, and a closed form of $w_{pi}(t+1)$ can be obtained at each iteration. Hence we can replace $w_{pi}$ terms in (\ref{equifp1}), (\ref{equiLamb_1_1}) and (\ref{equiLambd_1_2}) by its closed expression. Moreover, by initializing the dual variables  $\lambda_{pj}^k=\mathbf{0}_{d\times d}$, and letting $\lambda_{p}(t) =  \sum_{j\in \mathcal{N}_p}\lambda^k_{pj} $, $p\in \mathcal{P}$, $j\in \mathcal{N}_p$, $k=a$, $b$, we then can combine (\ref{equiLamb_1_1}) and (\ref{equiLambd_1_2}) into one update. As a result, the update procedures (\ref{equifp1})-(\ref{equiLambd_1_2}) can be further simplified through replacing $w_{pi}$ by its corresponding closed form in (\ref{equiDLag}). Moreover, if we initialize $ \lambda^a_{pj}(0)= \lambda^a_{jp}(0)=0$ for all $p\in\mathcal{P}$ and all $j\in\mathcal{N}$, and from (\ref{equiLamb_1_1})- (\ref{equiLambd_1_2}), it is clear that $ \lambda^a_{pj}(1)= \lambda^a_{jp}(1)$. Suppose that $\lambda^a_{pj}(t-1)= \lambda^a_{jp}(t-1)$, then by induction $\lambda^a_{pj}(t)= \lambda^a_{jp}(t)$. 
initializing the dual variables  $\lambda_{pj}^k=\mathbf{0}_{d\times d}$, and letting $\lambda_{p}(t) =  \sum_{j\in \mathcal{N}_p}\lambda^k_{pj} $, $p\in \mathcal{P}$, $j\in \mathcal{N}_p$, $k=a$, $b$, we can combine (\ref{equiLamb_1_1}) and (\ref{equiLambd_1_2}) into one update. Thus, we simplify (\ref{equifp1})-(\ref{equiLambd_1_2}) by introducing the following:
Let $L_p^{N}(t)$ be the short-hand notation of $L_p^{N}(\{f_{p}\},\{f_p(t)\}, \{\lambda_{p}(t)\})$ as :
\begin{equation}\label{equiLag_p}
\begin{aligned}
L_p^{N}(t):=&\frac{C^R}{B_p}\sum_{i=1}^{B_{p}}\mathcal{L}(y_{ip} f_{p}^T x_{ip})+\rho R(f_p) +2\lambda_{p}(t)^{T}f_{p} \\
&+\eta\sum_{i\in \mathcal{N}_p}\parallel f_{p}-\frac{1}{2}(f_p(t)+f_i(t)) \parallel^2.
\end{aligned}
\end{equation}
The ADMM iterations (\ref{equifp1})-(\ref{equiLambd_1_2}) can be reduced to
\begin{equation}\label{equifp_2}
f_{p}(t+1)= \arg\min_{f_{p}}L_p^N(f_{p},f_p(t), \lambda_{p}(t)),
\end{equation}
\begin{equation}\label{equiLamb2}
\lambda_{p}(t+1) = \lambda_{p}(t)+ \frac{\eta}{2}\sum_{j\in \mathcal{N}_p}[ f_{p}(t+1)-f_{j}(t+1)].
\end{equation}
%

% Algorithm 1
\begin{algorithm}[tb]
   \caption{Distributed ERM}
 
   \label{alg:DistributedERM}
\begin{algorithmic}
  {\small \STATE {\bfseries Required:} Randomly initialize $f_{p}, \lambda_{p} = \mathbf{0}_{d\times 1}$ for every $p\in\mathcal{P}$ 
   \STATE {\bfseries Input:} $\hat{D}$
   \FOR{$t = 0,1,2,3,...$}
   \FOR{$p=0$ {\bfseries to} $P$}
   \STATE {Compute $f_{p}(t+1)$ via (\ref{equifp_2}).}
   \ENDFOR
   \FOR{$p=0$ {\bfseries to} $P$}
   \STATE {Broadcast $f_{p}(t+1)$ to all neighbors $j\in\mathcal{N}_p$.}
   \ENDFOR
   \FOR{$p=0$ {\bfseries to} $P$}
   \STATE {Compute $\lambda_{p}(t+1)$ via (\ref{equiLamb2}).}
   \ENDFOR
   \ENDFOR
   \STATE{\bfseries Output:} $f^*$.}
\end{algorithmic}
\end{algorithm}

ADMM-based distributed ERM iterations (\ref{equifp_2})-(\ref{equiLamb2}) is and summarized in Algorithm 1. Every node $p\in\mathcal{P}$ updates its local $d\times1$ estimates $f_p(t)$ and $\lambda_p(t)$. At iteration $t+1$, node $p$ updates the local $f_p(t+1)$ through (\ref{equifp_2}). Next, node $p$ broadcasts the latest $f_p(t+1)$ to all its neighboring nodes $j\in\mathcal{N}_p$. Iteration $t+1$ finishes as each node updates the $\lambda_p(t+1)$ via (\ref{equiLamb2}).

Every iteration of our algorithm is still a minimization problem similar to the centralized problem (\ref{CRERM}). However, the number of variables participating in solving (\ref{equifp_2}) per node per iteration is $N_p$, which is much smaller than the one in the centralized problem, which is $\sum_{p=1}^P N_p$. There are several methods to solve (\ref{equifp_2} ). For instance, projected gradient method, Newton method, and Broyden-Fletcher-Goldfarb-Shanno (BFGS) method \cite{dai2013perfect} that approximates the Newton method, to name a few. 

ADMM-based distributed machine learning has benefits due to its high scalability. It also provides a certain level of privacy since nodes do not communicate data directly but their decision variable $f_p$. However, the privacy arises when an adversary can make intelligent inferences at each step and extract the sensitive information based on his observation of the learning output of his neighboring nodes. %is mainly due to the local parameter exchange among neighboring nodes instead of centralized communication.
 %Dually, the parameter of each node is anonymous to the non-neighboring nodes. However, the neighboring nodes can access to the parameter without privacy protection; also, as shown in Section 1, 
 Simple anonymization is not sufficient to address this issue as discussed in Section 1. In the following subsection, we will discuss the adversary models, and present differential privacy solutions. % Therefore  because it is still possible for adversary to extract the sensitive information with side information about the target. 

\subsection{Privacy Concerns}
Although the data stored at each node is not exchanged during the entire ADMM algorithm, the potential privacy risk still exists. Suppose that the dataset $D_p$ stored at node $p$ contains sensitive information in data point $(x_i, y_i)$ that is not allowed to be released to other nodes in the network or anyone else outside. Let $K:\mathbb{R}^d\rightarrow \mathbb{R}$ be the randomized version of Algorithm 1, and let $\{f^*_p\}_{p\in\mathcal{P}}$ be the output of $K$ at all the nodes. Then, the output $\{f^*_p\}_{p\in\mathcal{P}}$ is random.  In the distributed version of the algorithm, each node optimizes its local empirical risk based on its own dataset $D_p$. Let $K^t_p$ be the node-$p$-dependent stochastic sub-algorithm of $K$ at iteration $t$, and let $f_p(t)$ be the output of $K^t_p(D_p)$ at iteration $t$ inputing $D_p$. Hence the output $f_p(t)$ is stochastic at each $t$. In this work, we consider the following attack model. The adversary can access the learning outputs of intermediate ADMM iterations as well as the final output. This type of adversary aims to obtain sensitive information about the private data point of the training dataset by observing the output $f_p(t)$ of $K^t_p$ or $f^*_p$ of $K$ for all $p\in \mathcal{P}$ at every stage $t$ of the training.
We protect the privacy of distributed network using the definition of \textit{differential privacy} in  \cite{dwork2006calibrating}. Specifically, we require that a change of any single data point in the dataset might only change the distribution of the output of the algorithm slightly, which is visible to the adversary; this is done by adding randomness to the output of the algorithm. Let $D_p$ and $D'_p$ be two datasets differing in one data point; i.e., let $(x_{ip},y_{ip})\subset D_p$, and $(x'_{ip},y'_{ip})\subset D'_p$, then $(x_{ip},y_{ip})\neq (x'_{ip},y'_{ip})$. In other words, their \textit{Hamming Distance}, which is defined as 
%\begin{equation}\label{equHam}
$H_d(D_p, D'_p) = \sum_{i=0}^{B_p}\textbf{1}\{i:x_i\neq x'_i\}$,
%\end{equation} 
equals $1$; i.e., $H_d(D_p, D'_p) =1$. 

To protect the privacy against the adversary, we propose the concept of dynamic differential privacy, which enables the dynamic algorithm to be privacy-preserving at every stage of the learning.
\begin{mydef}(Dynamic $\alpha(t)$-Differential Privacy (DDP))
Consider a network of $P$ nodes $\mathcal{P}=\{1,\:2,\:...,\:P\}$, and each node $p$ has a training dataset $D_p$, and $\hat{D} = \bigcup_{p \in\mathcal{P}} D_{p}$.
Let $ K: \mathbb{R}^{d} \rightarrow \mathbb{R}$ be a randomized version of Algorithm 1.
Let $\alpha(t) = (\alpha_1(t), \alpha_2(t),..., \alpha_P(t))\in \mathbb{R}^P_{+}$, where $\alpha_p(t)\in \mathbb{R}_+$ is the privacy parameter of node $p$ at iteration $t$.
 Let $K^t_p$ be the node-$p$-dependent sub-algorithm of $K$, which corresponds to an ADMM iteration at $t$ that outputs $f_p(t)$.
Let $D'_p$ be any dataset with $H_d(D'_p, D_p)=1$, and $g_p(t)= K^t_p(D'_p)$.
We say that the algorithm $K$ is \emph{dynamic $\alpha_p(t)$-differential private} (DDP) if for any dataset $D'_p$, and for all $p\in\mathcal{P}$ that can be observed by an adversary of Type 2, and for all possible sets of the outcomes $S\subseteq \mathbbm{R}$, the following inequality holds:
\begin{equation}
\Pr[f_p(t)\in S]\leq e^{\alpha_p(t)} \cdot\Pr[g_p(t)\in S],
\end{equation}
for all  $t\in\mathbb{Z}$ during a learning process. The probability is taken with respect to $f_p(t)$, the output of $K^t_p$ at every stage $t$. The algorithm $K$ is called \emph{dynamic $\alpha(t)$-differential private} if the above conditions are satisfied.
\end{mydef}

Definition 1 provides a suitable differential privacy concept for the adversary. For dynamic $\alpha_p(t)$-differential private algorithms, the adversaries cannot extract additional information by observing the intermediate updates of $f_p(t)$ at each step.
Clearly, the algorithm with ADMM iterations shown in (\ref{equifp_2}) to (\ref{equiLamb2}) is not dynamic $\alpha_p(t)$-differential private. This is because the intermediate and final optimal output $f_p$'s are deterministic given dataset $D_p$. For $D'_p$ with $H_d(D_p, D'_p)=1$, the classifier will change completely, and the probability density $\Pr([f_p|D'_p])=0$, which leads to the ratio of probabilities $\frac{\Pr[f_p|D_p]}{\Pr[f_p|D'_p]}\rightarrow \infty$. 
In order to provide the DDP, we propose two algorithms, \textit{dual variable perturbation} and \textit{primal variable perturbation}, which are described in Section 3.1 and 3.2, respectively. 

\section{Dynamic Private Preserving}

\subsection{Dual Variable Perturbation}

In this subsection, we describe two algorithms that provide dynamic $\alpha$-differential privacy defined in Section 2.2. %In order to provide differential privacy defined in Definition 1 and 2, 
We protect the first algorithm based on  \textit{dual variable perturbation} (DVP), in which the dual variables $\{\lambda_p(t)\}_{p=1}^{P}$ are perturbed with a random noise vector $\epsilon_p(t)\in\mathbb{R}^d$ with the probability density function 
%15
%\begin{equation}\label{equiNoise}
$
\mathcal{K}_p(\epsilon) \sim e^{-\zeta_p(t) \parallel \epsilon \parallel},
$
%\end{equation}
where $\zeta_p(t)$ is a parameter related to the value of $\alpha_p(t)$, and $\parallel \cdot\parallel$ denotes the $l_2$ norm.  At each iteration, we first perturb the dual variable $\lambda_p(t)$, obtained from the last iteration, and store it in a new variable $\mu_p(t)=\lambda_p(t)+\epsilon_p(t)$. Now the corresponding node-$p$-based augmented Lagrange function $L_p^{N}(t)$ becomes $L_p^{dual}\big(f_{p},f_p(t), \mu_{p}(t+1),\{f_i(t) \}_{i\in\mathcal{N}_p}\big)$, defined as follows, and $L_p^{dual}(t)$ is used as a short-hand notation: %to denote $L_{dual}\big(f_{p},f_p(t), \mu_{p}(t+1),\{f_i(t) \}_{i\in\mathcal{N}_p}\big) $,
% 16

\begin{equation}\label{equiLagDual}
\begin{aligned}
L_p^{dual}(t)=& \frac{C^R}{B_p}\sum_{i=1}^{B_{p}}\mathcal{L}(y_{ip} f_{p}^T x_{ip})+\rho R(f_p) \\
&+2\mu_{p}(t+1)^{T}f_{p} +\frac{\Phi}{2}\parallel f_p \parallel^2 \\
&+\eta\sum_{i\in \mathcal{N}_p}\parallel f_{p}-\frac{1}{2}(f_p(t)+f_i(t)) \parallel^2,
\end{aligned}
\end{equation}
where $\frac{\Phi}{2}\parallel f_p \parallel^2$ is an additional penalty. As a result, the minimizer of $L_p^{dual}(t)$ is random. At each iteration, we first perturb the dual variable $\lambda_p(t)$, obtained from the last iteration, and store it in a new variable $\mu_p(t+1)$. 

Now, the iterations (\ref{equifp_2})-(\ref{equiLamb2}) becomes follows:

% 17
\begin{equation}\label{equiLamb_dual_pert}
\mu_{p}(t+1) = \lambda_{p}(t)+ \frac{C^R}{2B_p}\epsilon_p(t+1),
\end{equation}
% 18
\begin{equation}\label{equifp_Dual}
f_{p}(t+1) = \arg\min_{f_{p}}L_p^{dual}(t),
\end{equation}
% 19
\begin{equation}\label{equiLamb_dual}
\lambda_{p}(t+1) = \lambda_{p}(t)+ \frac{\eta}{2}\sum_{j\in \mathcal{N}_p}[ f_{p}(t+1)-f_{j}(t+1)].
\end{equation}

\begin{figure*}%
\centering
\begin{subfigure}{0.8\columnwidth}
\includegraphics[width=\columnwidth]{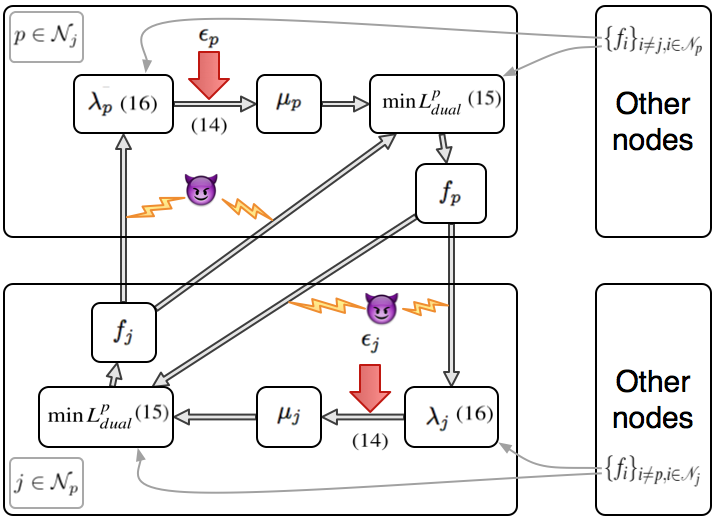}%
\caption{DVP during intermediate iterations}%
\label{subfigb}%
\end{subfigure}\hfill% 
\begin{subfigure}{0.8\columnwidth}
\includegraphics[width=\columnwidth]{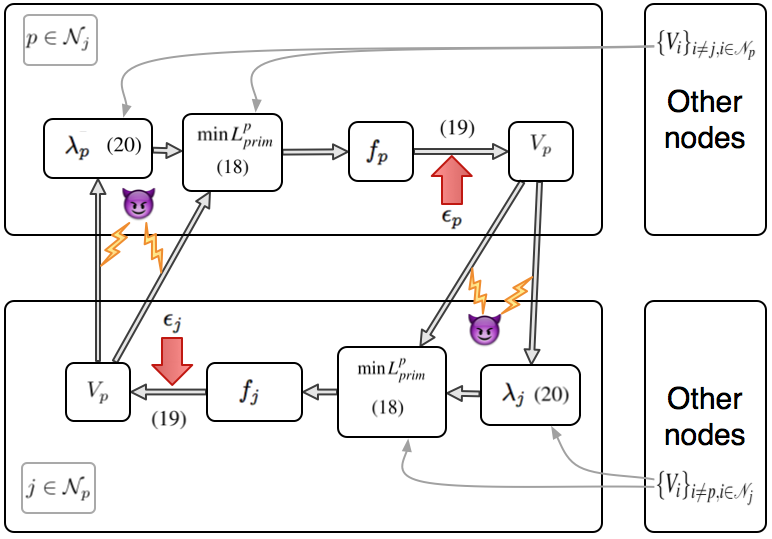}%
\caption{PVP during intermediate iterations}%
\label{subfigc}%
\end{subfigure}\hfill%
\caption{(a): DVP during intermediate iterations. The pertured $\mu_p$ participates in the $(15)$. As a result, the output $f_p$ at each iteration is a random variable, and the transmission of $f_p$ is differential private. (b): PVP during intermediate iterations. The pertured $V_p$ is a random variable. As a result, the transmission of $V_p$ is differential private.  }
\label{figabc}
\end{figure*}

% DVP Alg
\begin{algorithm}[tb]
   \caption{Dual Variable Perturbation}
   \label{alg:DVP}
\begin{algorithmic}
   {\small \STATE {\bfseries Required:} Randomly initialize $f_{p}, \lambda_{p} = \mathbf{0}_{d\times 1}$ for every $p\in\mathcal{P}$
   \STATE {\bfseries Input:} $\hat{D}$, $\{[\alpha_p(1),\alpha_p(2),...]\}_{p=1}^{P}$
   \FOR{$t = 0,1,2,3,...$}
   \FOR{$p=0$ {\bfseries to} $P$}
   \STATE {Let $\hat{\alpha}_p = \alpha_p(t) - \ln\Big(1+\frac{c_1}{ \frac{B_p}{C^R}\big(\rho+2\eta N_{p}\big)}\Big)^2$.}
   \IF{$\hat{\alpha}_p >0$}
   \STATE{$ \Phi = 0$.}
   \ELSE
   \STATE{$\Phi= \frac{c_1}{\frac{B_p}{C^R}(e^{\alpha_p(t)/4}-1)}-$
   $\rho-2\eta N_p$ and $\hat{\alpha}_p  = \alpha_p(t)/2$.}
   \ENDIF
   \STATE{Draw noise $\epsilon_{p}(t)$ according to $\mathcal{K}_p(\epsilon) \sim e^{-\zeta_p(t) \parallel \epsilon \parallel}$ with $\zeta_p(t) = \hat{\alpha}_p $.}
   \STATE{
    Compute $\mu_p(t+1)$  via (\ref{equiLamb_dual_pert}).}
    \STATE{
Compute $f_{p}(t+1)$ via (\ref{equifp_Dual}) with augmented Lagrange function as (\ref{equiLagDual}).}
   \ENDFOR
   \FOR{$p=0$ {\bfseries to} $P$}
   \STATE {Broadcast $f_{p}(t+1)$ to all neighbors $j \in \mathcal{N}_p$.}
   \ENDFOR
   \FOR{$p=0$ {\bfseries to} $P$}
   \STATE {Compute $\lambda_{p}(t+1)$ via (\ref{equiLamb_dual}).}
   \ENDFOR
   \ENDFOR
   \STATE{\bfseries Output:} $\{f_p^*\}_{p=1}^{P}$.}
\end{algorithmic}
\end{algorithm}

The iterations (\ref{equiLamb_dual_pert})-(\ref{equiLamb_dual}) are summarized as Algorithm 2, and are illustrated in Figure 1 and 3. All nodes have its corresponding value of $\rho$. Every node $p\in\mathcal{P}$ updates its local estimates $\mu_p(t)$, $f_p(t)$ and $\lambda_p(t)$ at time $t$; at time $t+1$, node $p$ first perturbs the dual variable $\lambda_p(t)$ obtained at time $t$ to obtain $\mu_p(t+1)$ via (\ref{equiLamb_dual_pert}), and then uses training dataset $D_p$ to compute $f_p(t+1)$ via (\ref{equifp_Dual}). Next, node $p$ sends $f_p(t+1)$ to all its neighboring nodes. The $(t+1)$-th update is done when each node updates its local $\lambda_p(t+1)$ via (\ref{equiLamb_dual}). We then have the following theorem.

\begin{figure}[htpb]
\includegraphics[scale=0.30]{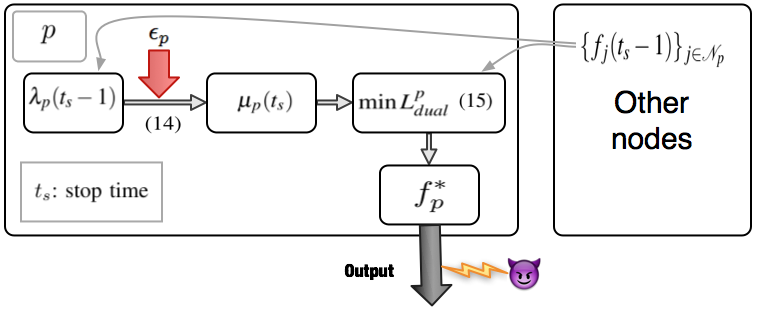}
\centering
\caption{The final iteration of both DVP and PVP. The pertured $\mu_p$ participates in the $(15)$. As a result, the output $f^*_p$ is a random variable, and the final output is differential private.}
\end{figure}

\begin{theorem} 
%Let $\hat{\alpha}_p  = \alpha_p(t) - \ln\Big(1+\frac{c_1}{ \frac{B_p}{C^R}\big(\rho+2\eta N_{p}\big)}\Big)^2$, and 
%$$$ 
%\Phi = \left\{ \begin{array}{ll}
%0 & \textrm{if~}\hat{\alpha}_p >0, \\
%\frac{c_1}{\frac{B_p}{C^R}(e^{\alpha_p(t)/4}-1)}-\rho-2\eta N_p  & \textrm{~otherwise}.
%\end{array} \right.
%$$
% If $\hat{\alpha}_p >0, $ then $ \Phi = 0$; else, let $\Phi= \frac{c_1}{\frac{B_p}{C^R}(e^{\alpha_p(t)/4}-1)}-\rho-2\eta N_p$, 
% {\bf and as a result $\hat{\alpha}_p  = \alpha_p(t)/2$.}
Under Assumption 1, 2 and 3, if the DR-ERM problem can be solved by Algorithm 2, then Algorithm $2$ solving this distributed problem is dynamic $\alpha$-differential private with $\alpha_p(t)$ for each node $p\in\mathcal{P}$ at time $t$. Let $Q(f_p(t)|D)$ and $Q(f_p(t)|D'_p)$ be the probability density functions of $f_p(t)$ given dataset $D$ and $D'_p$, respectively, with $H_d(D, D'_p)=1$. The ratio of conditional probabilities of $f_p(t)$ is bounded as follows:
% 20
\begin{equation}\label{equi_Theorem1_bound}
\frac{Q(f_p(t)|D)}{Q(f_p(t)|D'_p)}\leq e^{\alpha_p(t)}.
\end{equation}  

\textbf{Proof: See Appendix B.}
\end{theorem}

%
% THIS RESULT CAN APPEAR IN THE NEXT PAPER
%
%
%The definitions of differential privacy in Section 2 (and also in, for example, [4, 11, 12]) only consider the change of output distribution corresponding to a change of a single entry of the dataset. However, in many cases, the sensitive information may be contained in more than one data point. Actually, Theorem 1 can be extended to deal with the dataset that has multiple sensitive data entries.
%
%\begin{corollary}
%\textit{Let $D$ and $D'_p$ be two datasets with $H_d(D,D'_p)=c_2$, $c_2\geq 1$. Then an algorithm meets Theorem 1 can compute a $f_p(t)$ that has the following bounded ratio of conditional densities:
%\begin{equation}
%\frac{Q(f_p(t)|D)}{Q(f_p(t)|D'_p)}\leq e^{c_2\alpha'_p(t)}, 
%\end{equation} }
%
%\textbf{Proof: See Appendix C}.
%
%\end{corollary}
%
%
%However, the output distribution must change corresponging to the a change of multiple data entries; as a result, the level of privacy has to decrease, especially for large $c_2$ and large $B_P$.

\subsection{Primal Variable Perturbation}

In this subsection, we provide the algorithm based on the \textit{primal variable perturbation} (PVP), which perturbs the primal variable $\{f_p(t+1)\}_{p=0}^P$ before sending the decision to the neighboring nodes. This algorithm can also provide dynamic differential privacy defined in Definition 1 and 2. Let the node-$p$-based augmented Lagrange function $L_p^{prim}\big(f_{p}, f_p(t),\epsilon_p(t),\lambda_{p}(t), \{V_i(t)\}_{i\in\mathcal{N}_p}\big)$ be defined as follows, and use $L_p^{prim}(t)$ as its short hand notation:
$$
\begin{aligned}
L_p^{prim}(t)=&\frac{C^R}{B_p}\sum_{i=1}^{B_{p}}\mathcal{L}(y_{ip} f_{p}^T x_{ip})+\rho R(f_p) +2\lambda_{p}(t)^{T}f_{p} \\
&+\eta\sum_{i\in \mathcal{N}_p}\parallel f_{p}-\frac{1}{2}(f_p(t)+V_i(t)-\epsilon_p(t)) \parallel^2.
\end{aligned}
$$
In this method, we divide the entire training process into two parts: (i) the intermediate iterations, and (ii) the final interation. 
During the intermediate iterations, we use the unperturbed primal $f_p(t)$ obtained at time $t$ in the augmented Lagrange function and subtract the noise vector $\epsilon_p(t)$ added at time $t$ to reduce the noise in the minimization in (\ref{equiVp}). Note that the noise $\epsilon_p(t)$ at time $t$ is known at time $t+1$. The privacy of releasing primal variable is not affected.

%The following iterations specify the corresponding ADMM iterations.
The corresponding ADMM iterations  that can provide dynamic $\alpha_p(t)$-differential privacy at time $t$ are as follows:
%21
\begin{equation}\label{equifp_prim}
f_{p}(t+1) = \arg\min_{f_{p}}L_p^{prim}(t),
\end{equation}
%22
\begin{equation}\label{equiVp}
V_{p}(t+1) = f_{p}(t+1)+ \epsilon_p(t+1),
\end{equation}
%23
\begin{equation}\label{equiLamb_prim}
\lambda_{p}(t+1) = \lambda_{p}(t)+ \frac{\eta}{2}\sum_{j\in \mathcal{N}_p}[ V_{p}(t+1)-V_{j}(t+1)],
\end{equation}
where $\epsilon_p(t+1)$ is the random noise vector with the density function $\mathcal{K}_p(\epsilon) \sim e^{-\zeta_p(t) \parallel \epsilon \parallel}$. The  augmented Lagrange function is (\ref{equiLag_p}). 
Let $t_s$ be the time when we enter the final iteration.
When $t=t_s$ we enter the final iteration at $t_s$, we apply the DVP to update the variables.
Specifically, we input the data sets $\hat{D}$ to DVP and use the $\{f_p(t_s-1)\}_p$ and $\{\lambda_p(t_s-1)\}_p$, obtained from (\ref{equifp_prim}) and (\ref{equiLamb_prim}), in iteration (\ref{equiLamb_dual_pert})-(\ref{equiLamb_dual}):
\begin{equation}\label{equiPrim_Last_up}
\mu_{p}(t_s+1) = \lambda_{p}(t_s)+ \frac{C^R}{2B_p}\epsilon_p(t_s+1),
\end{equation}    
\begin{equation}\label{equiPrim_Last_fp}
f_{p}(t+1) = \arg\min_{f_{p}}L_p^{dual}(t_s),
\end{equation}
\begin{equation}\label{equiPrim_Last_lambda_p}
\lambda_{p}(t_s+1) = \lambda_{p}(t_s)+ \frac{\eta}{2}\sum_{j\in \mathcal{N}_p}[ f_{p}(t_s+1)-f_{j}(t_s+1)].
\end{equation}
$\{f_p(t_s+1)\}_{p\in\mathcal{P}}$ is the final output of the PVP algorithm. 

The iterations (\ref{equifp_prim})-(\ref{equiLamb_prim}) and (\ref{equiPrim_Last_up})-(\ref{equiPrim_Last_lambda_p}) are summarized in Algorithm 3, and are illustrated in Figure 4 and 5. Each node $p\in\mathcal{P}$ updates $f_p(t)$, $V_p(t)$ and $\lambda_p(t)$ at time $t$. Then, at time $t+1$, the training dataset is used to compute $f_p(t+1)$ via  (\ref{equifp_prim}), which is then perturbed to obtain $V_p(t+1)$ via (\ref{equiVp}). Next, $V_p(t+1)$ is distributed to all the neighboring nodes of node $p$. Finally, $\lambda_p(t+1)$ is updated via (\ref{equiLamb_prim}). The final iteration follows the DVP. We then have the following theorem.

% PVP Algorithm
\begin{algorithm}[tb]
   \caption{Primal Variable Perturbation}
   \label{alg:PVP}
\begin{algorithmic}
   {\small \STATE {\bfseries Required:} Randomly initialize $f_{p}, \lambda_{p} = \mathbf{0}_{d\times 1}$ for every $p\in\mathcal{P}$
   \STATE {\bfseries Input:} $\hat{D}$, $\{[\alpha_p(1),\alpha_p(2),...]\}_{p=1}^{P}.$
   \FOR{$t = 0,1,2,3,...$}
   \FOR{$p=0$ {\bfseries to} $P$}
   \STATE { Draw noise $\epsilon_{p}(t)$ according to $\mathcal{K}_p(\epsilon) \sim e^{-\zeta_p(t) \parallel \epsilon \parallel}$ with $\zeta_p(t) =  \frac{\rho B_p\alpha_p(t)}{2C^R}$.}
   \STATE{Compute $f_{p}(t+1)$ via (\ref{equifp_prim})
   with augmented Lagrange function as (\ref{equiLag_p}).}
  \STATE{
Compute $V_{p}(t+1)$ via (\ref{equiVp}).}
   \ENDFOR
   \FOR{$p=0$ {\bfseries to} $P$}
   \STATE {Broadcast $f_{p}(t+1)$ to all neighbors $j \in \mathcal{N}_p$.}
   \ENDFOR
   \FOR{$p=0$ {\bfseries to} $P$}
   \STATE {Compute $\lambda_{p}(t+1)$ via (\ref{equiLamb_prim}).}
   \IF {$t=$ \textit{stop time}}
   \STATE {Use the latest $\{f_p(t)\}_p$ and $\{\lambda_p(t)\}_p$ obtained as initial values, and input $\hat{D}$ to Algorithm 2 to iterate the loop once.}
   \ENDIF
   \ENDFOR
   \ENDFOR
   \STATE{\bfseries Output:} $\{f_p^*\}_{p=1}^{P}$.}
\end{algorithmic}
\end{algorithm}

\begin{theorem}
Under Assumption 1, 2 and 3, if the DR-ERM problem can be solved by Algorithm 3, then Algorithm 3 solving this distributed problem is dynamic $\alpha(t)$-differential private. The ratio of conditional probabilities of $f_p(t)$ is bounded as in (\ref{equi_Theorem1_bound}).

\textbf{Proof: See Appendix C}.
\end{theorem}

%\begin{corollary}
%\textit{Let $D$ and $D'_p$ be two datasets with $H_d(D,D'_p)=c_2$, $c_2\geq 1$. Any algorithm satisfies Theorem 2 can produce a private $f_p(t)$, which has the following bounded ratio of condtional densities at each iteration:
%\begin{equation}
%\frac{Q(f_p(t)|D)}{Q(f_p(t)|D'_p)}\leq e^{c_2\alpha'_p(t)} .
%\end{equation} }
%\end{corollary}
%The proof of Corollary 2.2 is the same as that of Corollary 1.2 in Appendix C.
%

\section{Performance Analysis}

In this section, we discuss the performance of Algorithm 2 and 3. We establish performance bounds for regularization functions with $l_2$ norm. 
Our analysis is based on the following assumptions:

\begin{asu} 
The data points $\{(x_{pi}, y_{pi})\}_{i=1}^{B_p}$ are drawn i.i.d. from a fixed but unknown probability distribution $\mathbbm{P}^{xy}(x_{pi}, y_{pi})$ at each node $p\in\mathcal{P}$. 
\end{asu} 
\begin{asu} 
$\epsilon_p(t)$ is drawn from (15) with the same $\alpha_p(t)=\alpha(t)$ for all $p\in\mathcal{P}$ at time $t\in\mathbb{Z}.$ 
\end{asu} 
We then define the expected loss of node $p$ using classifier $f_p$ as follows, under Assumption 4:
$
\hat{C}(f_p) :=C^R\mathbbm{E}_{(x,y)\sim \mathbbm{P}^{xy}}(\mathcal{L}(yf^Tx)),
$
and the corresponding expected objective function $\hat{Z}$ is:
$
\hat{Z}_p(f_p) := \hat{C}(f_p) + \rho R(f_p).
$
The performance of non-private non-distributed ERM classification learning has been already studied by, for example,  Shalev et al. in \cite{shalev2008svm} (also see the work of Chaudhuri  et al. in \cite{chaudhuri2011differentially}), which introduces a reference classifier $f^0$ with expected loss $ \hat{C}(f^0)$, and shows that if the number of data points is sufficiently large, then the actual expected loss of the trained $l_2$ regularized support vector machine (SVM) classifier $f_{SVM}$ satisfies
%\begin{equation}\label{section4_1}
$
\hat{C}(f_{SVM})\leq \hat{C}^0 + \alpha_{acc},
$
%\end{equation}
where $\alpha_{acc}$ is the generalization error. We use a similar argument to study the accuracy of Algorithm 1. Let $f^0$ be the reference classifier of Algorithm 1. We quantify the performance of our algorithms with $f^*$ as the final output by the number of data points required to obtain  
$
\hat{C}(f^*) \leq \hat{C}^0+\alpha_{acc}.
$

However, instead of focusing on only the final output, we care about the learning performance at all iterations.
Let $f_p^{non}(t+1) = \arg\min_{f_p}L_p^N(t)$ be the intermediate updated classifier at $t$, and let $f^*=\arg\min_{f_p}Z_p(f_p|D_p)$ be the final output of Algorithm 1. From Theorem 9 (see Appendix A), the sequence $\{f_p^{non}(t)\}$ is bounded and converges to the optimal value $f^*$ as time $t\rightarrow \infty$. Note that $\{f_p^{non}(t)\}$ is a non-private classifier without added perturbations. Since the optimization is minimization, then there exists a constant $\Delta^{non}(t)$ at time $t$ such that:
$
\hat{C}(f_p^{non}(t)) - \hat{C}(f^*) \leq \Delta^{non}(t),
$ 
and substituting it to %
$
\hat{C}(f^*) \leq \hat{C}^0+\alpha_{acc},
$
yields:
\begin{equation}\label{section4_2}
\hat{C}(f_p^{non}(t)) \leq \hat{C}^0+\Delta^{non}(t) +\alpha_{acc}.
\end{equation}
Clearly, the above condition depends on the reference classifier $f^0$; actually, as shown later in this section, the number of data points depends on the $l_2$-norm $\parallel f^0 \parallel$ of the reference classifier. Usually, the reference classifier is chosen with an upper bound on $\parallel f^0 \parallel$, say $b^0$. Based on (\ref{section4_2}), we provide the following theorem about the performance of Algorithm 1.

% Theorem 3: Algorithm 1
\begin{theorem}
Let $R(f_p(t))=\frac{1}{2}\parallel f_p(t)\parallel^2$, and let $f^0$ such that $\hat{C}(f^0)=\hat{C}^0$ for all $p\in\mathcal{P}$ at time $t$, and $\delta >0$ is a positive real number. Let $f_p^{non}(t+1) = \arg\min_{f_p} L_p^N(f_p,t|D_p)$ be the output of Algorithm 1. If Assumption 1 and 4 are satisfied, then there exists a constant $\beta_{non}$ such that if the number of data points, $B_p$ in $D_p=\Big\{(x_{ip},y_{ip})\subset \mathbbm{R}^d \times \{-1,1\}\Big\}$ satisfy:
$
B_p>\beta_{non}\Bigg( \frac{C^R\parallel f^0 \parallel^2 \ln(\frac{1}{\delta}) }{\alpha_{acc}^2}  \Bigg),
$
then $f_p^{non}(t+1)$ satisfies:
$
\mathbbm{P}\big( \hat{C}(f_p^{non}(t+1))\leq \hat{C}^0+\alpha_{acc}+\Delta^{non}(t) \big)\geq 1-\delta.
$
for all $t\in\mathbb{Z}_+$.

\textbf{Proof: See Appendix D}.
\end{theorem} 
Note that $\alpha_{acc}\leq 1$ is required for most machine learning algorithms. In the case of SVM, if the constraints are $y_i f^Tx_i\leq c_{SVM}$, for $i = 1,\:,...,\:n$, where $n$ is the number of data points, then, classification margin is $c_{svm}/\parallel f \parallel$. Thus, if we want to maximization the margin $c_{SVM}/\parallel f^0 \parallel$ we need to choose large value of $\parallel f^0 \parallel$. Larger value of $\parallel f^0 \parallel$ is usually chosen for non-separable or with small margin. 
In the following section, we provide the performance guarantees of Algorithm 2 and 3. 

\subsection{Performance of Private Algorithms}
Similar to Algorithm 1, we solve an optimization problem minimizing $L_p^{dual}(f_p,t|D_p)$ at each iteration. Let $f_p(t)$ and $\lambda_p(t)$ be the primal and dual variables used in minimizing $L_p^{dual}(f_p,t|D_p)$ at iteration $t$, respectively. Suppose that starting from iteration $t$, the noise vector is static with $\epsilon_p(t)$ generated at iteration $t$. To compare our private classifier at iteration $t$ with a private reference classifier $f^0(t)$, we construct a corresponding algorithm, Alg-2, associated with Algorithm 2. However, starting from iteration $t+1$, the noise vector in Alg-2 $\epsilon_p(t')=\epsilon_p(t)$ for all $t'>t$. In other words, solving Alg-2 is equivalent to solving the optimization problem with the objective function $Z_p^{dual}(f_p,t|D_p, \epsilon^{pi}(t))$, $t\geq 0$ defined as follows:
$$
\begin{aligned}
Z_p^{dual}(f_p,t|D_p, \epsilon_p(t)) :=Z_p(f_p|D_p)+\frac{C^R}{B_p}\epsilon_p(t)f_p.
\end{aligned}
$$ 
Let $f'_p(t)$ and $\lambda'_p(t)$ be the updated variables of the ADMM-based algorithm minimizing $ Z_p^{dual}(f_p,t|D_p)$ at iteration $t$.
Then, Alg-2 can be interpreted as minimizing $Z_p^{dual}(f_p,t|D_p, \epsilon^{pi}(t))$ with initial condition as $f'_p(0)=f_p(t)$ and $\lambda'_p(0)=\lambda_p(t)$ for all $p\in\mathcal{P}$. Let $Z_p^{dual}(f_p,t|D_p, \epsilon^{pi}(t))$ be regarded as the associated objective function of Alg-2.

For PVP, we can also introduce a similar algorithm denoted as Alg-3. Let $\epsilon^{pi}(t) = \epsilon_p(t)-\epsilon_i(t)$, for $i\in\mathcal{N}_p$. Then, the associated objective function of Alg-3 denoted by $Z_p^{prim}(f_p,t|D_p, \epsilon^{pi}(t))$, $t\geq0$, is defined as follows:
$$
\begin{aligned}
Z_p^{prim}(f_p,t|D_p, \epsilon^{pi}(t)) &:= Z_p(f_p|D_p)\\
&\hspace{-25mm}-\eta \sum_{i\in \mathcal{N}_p}\Big((f_p -\frac{1}{2}(f_p(t)+f_i(t))^T\cdot(\epsilon^{pi}(t))+\frac{1}{4}\big(\epsilon^{pi}(t)\big)^2\Big).
\end{aligned}
$$
Since both $Z_p^{dual}(f_p,t|D_p, \epsilon^{pi}(t))$ and $Z_p^{prim}(f_p,t|D_p, \epsilon^{pi}(t))$ are real and convex, then, similar to Algorithm 1, the sequence $\{f_p(t)\}$ is bounded and $f_p(t)$ converges to $f^*_p(t)$, which is a limit point of $f_p(t)$. Thus, there exists a constant  $\Delta^{priv}_p(t) = \Delta^{dual}_p(t)$ or $\Delta^{prim}_p(t)$ given noise vector $\epsilon_p(t)$ such that
%\begin{equation}\label{nonCost_1}
$
\hat{C}(f_p(t)) - \hat{C}(f^*_p(t))\leq \Delta^{priv}_p(t).
$
%\end{equation}
The performance analysis in Theorem 3 can also used in DVP and PVP. Specifically, the performance is measured by the number of data points, $B_p$, for all $p\in \mathcal{P}$ required to obtain 
$
\hat{C}(f_p(t)) \leq \hat{C}^0(t)+\alpha_{acc}+\Delta^{priv}_p(t).
$
We say that every learned $f_p(t)$ is $\alpha_{acc}$-optimal if it satisfies the above inequality.

Since in Alg-3, the perturbed primal variable $V_p(t')$ is equal to $f_p(t')$ plus a constant $\epsilon_p(t)$ generated by Algorithm 3 at iteration $t$, for $t'\geq 0 $, we can find a constant $\Delta_p^{primV}(t)$ such that
$
\hat{C}(V_p(t))-\hat{C}(V_p^*(t))\leq \Delta_p^{primV}(t).
$
Similarly, we measure the performance of $V_p$ by the number of data points, $B_p$, for all $p\in\mathcal{P}$ required to achieve
$
\hat{C}(V_p(t)) \leq \hat{C}^0(t)+\alpha_{acc}+\Delta^{primV}_p(t),
$
where $\hat{C}^0(t)=\hat{C}(f^0(t))$,and $f^0(t)$ is a reference classifier.

We now establish the performance bounds for Algorithm 2, DVP, which is summarized in the following theorem.

\begin{theorem}
Let $R(f_p(t))=\frac{1}{2}\parallel f_p(t)\parallel^2$, and $f^0_p(t)$ such that $\hat{C}(f^0_p(t))=\hat{C}^0(t)$ for all $p\in\mathcal{P}$, and a real number $\delta >0$. If Assumption 1, 4 and 5 are satisfied, then there exists a constant $\beta_{dual}$ such that if the number of data points, $B_p$ in $D_p=\Big\{(x_{ip},y_{ip})\subset \mathbbm{R}^d \times \{-1,1\}\Big\}$ satisfy:
{\small $$
\begin{aligned}
B_p > \beta_{dual}\max \Bigg( & \max_t \Big( \frac{\parallel f^0_p(t+1) \parallel d \ln(\frac{d}{\delta})}{\alpha_{acc}\alpha_p(t)}\Big),\\
& \hspace{-25mm} \max_t \Big(\frac{C^Rc_1\parallel f_p^0(t+1) \parallel^2}{\alpha_{acc}\alpha_p(t)}\Big), \max_t \Big(\frac{C^R\parallel f^0_p(t+1) \parallel^2 \ln(\frac{1}{\delta}) }{\alpha_{acc}^2}\Big) \Bigg),
\end{aligned}
$$}
then $f^*_p(t+1)$ satisfies:
$
\mathbbm{P}\big( \hat{C}(f^*_p(t+1))\leq \hat{C}^0(t+1)+\alpha_{acc} \big)\geq 1-2\delta.
$

\textbf{Proof: See Appendix E}.
\end{theorem}
\begin{corollary}
\textit{Let $f_p(t+1) = \arg\min_{f_p} L_p^{dual}(f_p,t|D_p)$ be the updated classifier of Algorithm 2 and let $f^0_p(t)$ be a reference classifier such that $\hat{C}(f^0_p(t)=\hat{C}^0(t)$. If all the conditions of Theorem 3 are satisfied,
then $f_p(t+1)$ satisfies
\begin{equation}\label{Coro_DVP_Conv}
\mathbbm{P}\big(\hat{C}(f_p(t+1))\leq \hat{C}^0(t)+\alpha_{acc}+\Delta_p^{dual}(t) \big)\geq 1-2\delta.
\end{equation}}
\end{corollary}
\begin{proof}

$
\hat{C}(f_p(t)) - \hat{C}(f^*_p(t))\leq \Delta^{dual}_p(t).
$
holds for $f_p(t)$ and $f^*_p(t)$
and from Theorem 3, 
$
\mathbbm{P}\big( \hat{C}(f^*_p(t+1))\leq \hat{C}^0(t+1)+\alpha_{acc} \big)\geq 1-2\delta.
$
Therefore, we can have (\ref{Coro_DVP_Conv}).
\end{proof}
Theorem 4 and Corollary 4.1 can guarantee the privacy defined in both Definition 1 and 2. The following theorem is used to analyze the performance bound of classifier $f_p(t+1)$ in (\ref{equifp_prim}), which minimizes $L_p^{prim}(t)$ that involves noise vectors from $V_p(t)$ perturbed at the previous iteration.

\begin{theorem}
Let $R(f_p(t))=\frac{1}{2}\parallel f_p(t)\parallel^2$, and $f^0_p(t)$ such that $\hat{C}(f^0_p(t))=\hat{C}^0(t)$, and a real number $\delta >0$. From Assumption 1, we have the loss function $\mathcal{L}(\cdot)$ is convex and differentiable with $\mathcal{L}'(\cdot)\leq 1$.
If Assumption 4 and 5 are satisfied, then there exists a constant $\beta^{A}_{prim}$ such that if the number of data points, $B_p$ in $D_p=\Big\{(x_{ip},y_{ip})\subset \mathbbm{R}^d \times \{-1,1\}\Big\}$ satisfies:
{\small$$
\begin{aligned}
B_p > \beta^A_{prim}\max\Bigg(&\max_t\Big(\frac{C^R\parallel f^0_p(t+1) \parallel^{3}\eta N_pd\ln(\frac{d}{\delta})}{\alpha^2_{acc}\alpha_p(t)}\Big),\\
&\max_t\Big(\frac{C^R\parallel f^0_p(t+1) \parallel^2 \ln(\frac{1}{\delta}) }{\alpha_{acc}^2}\Big) \Bigg),
\end{aligned}
$$
then $f^*_p(t+1)$ satisfies
$
\mathbbm{P}\big( \hat{C}(f^*_p(t+1))\leq \hat{C}^0(t+1)+\alpha_{acc} \big)\geq 1-2\delta.
$}
\textbf{Proof: See Appendix F}.
\end{theorem}

Next, we establish the PVP performance bound of Algorithm 3. Theorem 6 and Corollary 6.1 shows the requirements under which the performance of the part 1 of Algorithm 3 is guaranteed. Corollary 6.2 combines the results from Theorem 5 and Corollary 6.2 to provide the performance bound of the part 2 of Algorithm 3.
\begin{theorem} 
Let $R(f_p(t))=\frac{1}{2}\parallel f_p(t)\parallel^2$, and $f^0_p(t)$ such that $\hat{C}(f^0_p(t))=\hat{C}^0(t)$, and $\delta >0$ is a positive real number. Let $f^*_p(t+1)=\arg\min_{f_p} Z_p^{prim}(t)$ be $\alpha_{acc}$-accurate according to Theorem 4.
In addition to Assumption 1, we also assume that $\mathcal{L}'$ satisfies: 
$
|\mathcal{L}'(a)-\mathcal{L}'(b)|\leq c_4|a - b|
$ for all pairs $(a,b)$ with a constant $c_4$. If Assumption 4 and 5 are satisfied, then there exists a constant $\beta^B_{prim}$ such that if the number of data points, $B_p$ in $D_p=\Big\{(x_{ip},y_{ip})\subset \mathbbm{R}^d \times \{-1,1\}\Big\}$ satisfies:
{\small \begin{equation} \label{Theo_6}
\begin{aligned}
B_p >& \beta^B_{prim}\max\Bigg(\max_t\Big(\frac{C^R\parallel f^0_p(t+1) \parallel^{3}\eta N_pd\ln(\frac{d}{\delta})}{\alpha^2_{acc}\alpha_p(t)}\Big), \\
&\max_t\Big(\frac{C^R\parallel f^0_p(t+1) \parallel^2 \ln(\frac{1}{\delta}) }{\alpha_{acc}^2}\Big),\\
& \max_t\Big(\frac{4C^B \parallel f^0(t+1) \parallel d \big(\ln(\frac{d}{\delta}) \big)^2 }{\alpha_{acc}\alpha_p(t)}\Big),
\end{aligned}
\end{equation}}
{\small$$
\begin{aligned}
&\max_t\Big(\frac{4\parallel f^0_p(t+1) \parallel^{3}\eta N_pd\ln(\frac{d}{\delta})}{\alpha^2_{acc}\alpha_p(t)}\Big),\\
& \max_t\Big(\frac{4\big(C^R \big)^{\frac{3}{2}}\parallel f^0_p(t+1) \parallel^2d \ln(\frac{d}{\delta})  }{\alpha_{acc}^{3/2}\alpha_p(t)} \Big)  \Bigg),
\end{aligned}
$$}

then $V^*_p(t+1)=f^*_p(t+1)+\epsilon_p(t+1)$ satisfies
$
\mathbbm{P}\big( \hat{C}(V^*_p(t+1))\leq \hat{C}^0(t+1)+\alpha_{acc} \big)\geq 1-3\delta.
$

\textbf{Proof: See Appendix G}.
\end{theorem}

%%%%%%%%%%%%%%%%%%%%%%%%%
% COnvergence
\begin{figure*}%
\centering
\begin{subfigure}{.5\columnwidth}
\includegraphics[width=\columnwidth]{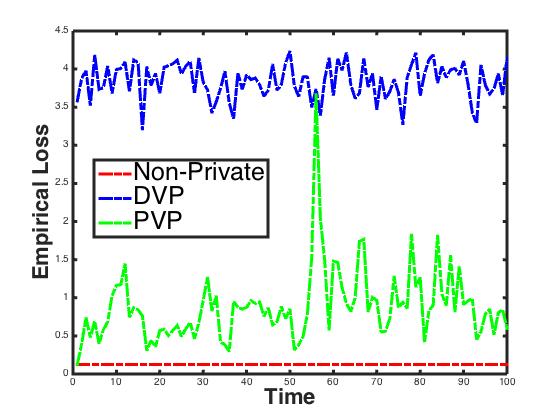}%
\caption{$\alpha_p(t) = 0.01$}%
\label{subfiga}%
\end{subfigure}\hfill%
\begin{subfigure}{.5\columnwidth}
\includegraphics[width=\columnwidth]{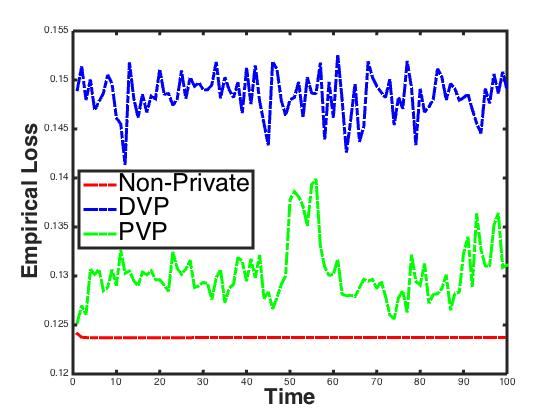}%
\caption{$\alpha_p(t) = 0.1$}%
\label{subfigb}%
\end{subfigure}\hfill%
\begin{subfigure}{.5\columnwidth}
\includegraphics[width=\columnwidth]{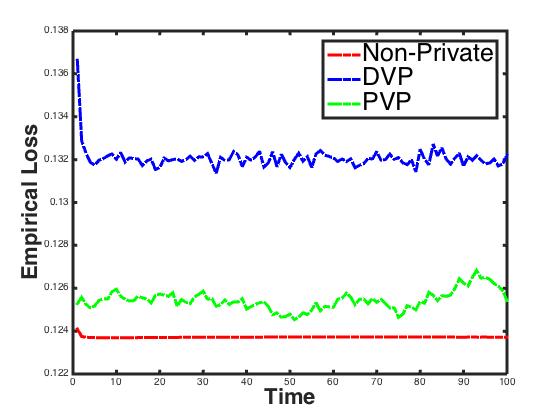}%
\caption{$\alpha_p(t) = 0.5$}%
\label{subfigc}%
\end{subfigure}\hfill%
\begin{subfigure}{.5\columnwidth}
\includegraphics[width=\columnwidth]{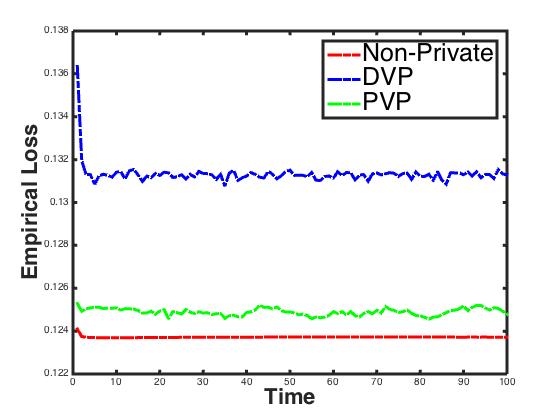}%
\caption{$\alpha_p(t) = 1$}%
\label{subfigc}%
\end{subfigure}\hfill%
\caption{Convergence of algorithms, at iteration $t=100$ (before the stop time) with different values of $\alpha_p(t)$. DVP with $\rho = 10^{-2.5}$ and $C^R=1750$; PVP with $\rho = 10^{-1}$ and $C^R=146$; Algorithm 1 (non-private) with $\rho = 10^{-10}$ and $C^R=1750$.}
\label{figabc}
\end{figure*}

\begin{figure*}%
\centering
\begin{subfigure}{.5\columnwidth}
\includegraphics[width=\columnwidth]{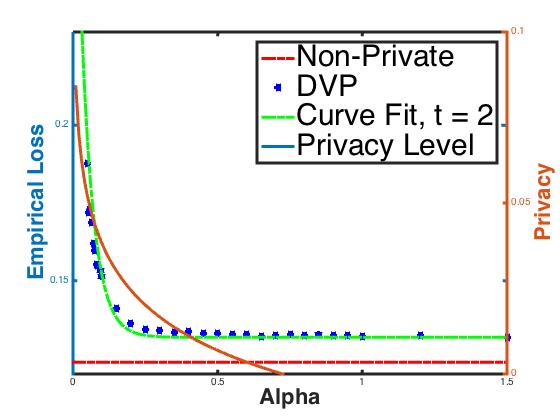}%
\caption{DVP: $t = 2$}%
\label{subfigb}%
\end{subfigure}\hfill% 
\begin{subfigure}{.5\columnwidth}
\includegraphics[width=\columnwidth]{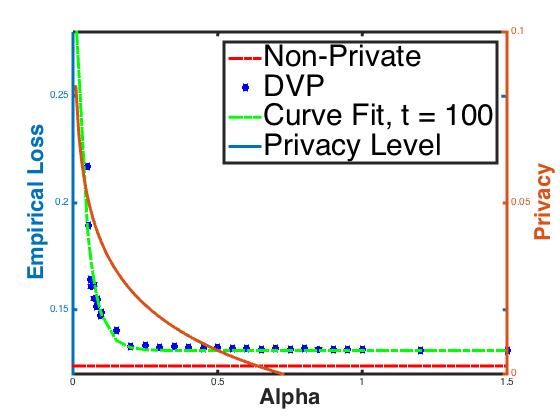}%
\caption{DVP: $t = 100$}%
\label{subfigc}%
\end{subfigure}\hfill%
\begin{subfigure}{.5\columnwidth}
\includegraphics[width=\columnwidth]{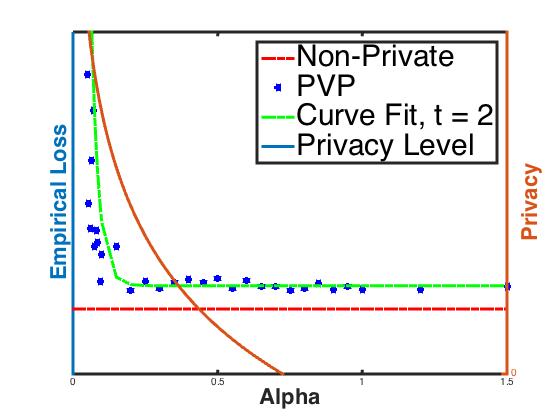}%
\caption{PVP: $t = 2$}%
\label{subfiga}%
\end{subfigure}\hfill%
\begin{subfigure}{.5\columnwidth}
\includegraphics[width=\columnwidth]{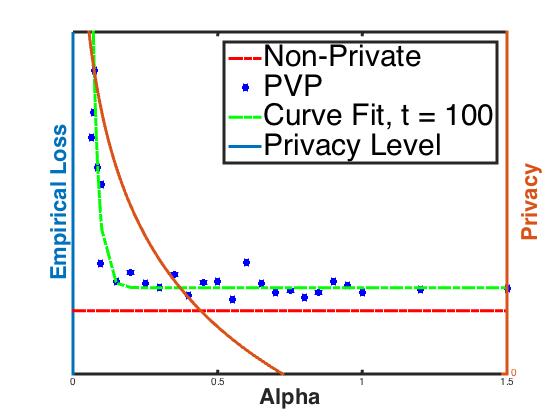}%
\caption{PVP: $t = 100$}%
\label{subfigc}%
\end{subfigure}\hfill%
\caption{Privacy-accuracy tradeoff.$(a)$-$(b)$: DVP, with $\omega_{p1} = 0.02$, $\omega_{p2}=6$, $\omega_{p3} = 9$, $\omega_{p4} = 1$ (before the stop time); $(c)$-$(d)$ PVP with $\omega_{p1} = 0.02$, $\omega_{p2}=6$, $\omega_{p3} = 9$, $\omega_{p4} = 1$ (before the stop time). }
\label{figabc}
\end{figure*}

%%%%%%%%%%%%%%%%%%%%%%%%%

%
\begin{corollary}
\textit{Let $f_p(t+1) = \arg\min_{f_p} L_p^{prim}(f_p,t|D_p)$ be the updated classifier of Algorithm 3, and let $f^0_p(t)$ be a reference classifier such that $\hat{C}(f^0_p(t)=\hat{C}^0(t)$. If all the conditions of Theorem 5 are satisfied,
then, $V_p(t+1) = f_p(t+1)+\epsilon_p(t+1)$ satisfies
\begin{equation}\label{Coro_PVP_Conv}
\mathbbm{P}\big(\hat{C}(V_p(t+1))\leq \hat{C}^0(t)+\alpha_{acc}+\Delta_p^{primV}(t) \big)\geq 1-3\delta.
\end{equation}}
\end{corollary}

\begin{proof} 
From Theorem 6, $V_p^*(t+1)$ satisfies
$
\mathbbm{P}\big( \hat{C}(V^*_p(t+1))\leq \hat{C}^0(t+1)+\alpha_{acc} \big)\geq 1-3\delta,
$
and since 
$
\hat{C}(V_p(t+1))-\hat{C}(V_p^*(t+1))\leq \Delta_p^{primV}(t+1),
$
then, we have (\ref{Coro_PVP_Conv}).

\end{proof}
\begin{corollary}
\textit{Let $f^*_p$ be the final output classifier of Algorithm 3 at node $p$, and let $f^0_p(t)$ be a reference classifier such that $\hat{C}(f^0_p(t)=\hat{C}^0(t)$. If all the conditions of Theorem 4 and 6 are satisfied,
then, $f^*_p$ satisfies
$$
\mathbbm{P}\big(\hat{C}(f^*_p)\leq \hat{C}^0(t)+\alpha_{acc}+\Delta_p^{dual}(t) \big)\geq 1-5\delta.
$$}
\end{corollary}
\begin{proof} 
All the conditions of Theorem 6 are satisfied to guarantee the privacy during the intermediate iterations. All the conditions of Theorem 4 are satisfied so that the final update is differential privacy is provided. Combining Theorem 4 and 6 yields the results.
\end{proof}
From Theorem 4 and 6, we can see that, for non-separable problems or ones  with a small margin, in which a larger $\parallel f_p^0(t)\parallel$ is used, the terms $\frac{1}{\alpha_{acc}}$ and $\parallel f_p^0(t)\parallel$ have a more significant influence on the requirement of datasets size for DVP than the PVP. Also, the performance of DVP is guaranteed with higher probability than PVP. Therefore, DVP is preferred for more difficult problems. 
Moreover, the privacy increases by trading the accuracy. It is essential to manage the tradeoff between the privacy and the accuracy, and this will be discussed in Section 5.

\section{Numerical Experiment}

In this section, we test Algorithm 2 and 3 with real world training dataset. The dataset used is the \textit{Adult} dataset from UCI Machine Learning Repository \cite{asuncion2007uci}, which contains demographic information such as age, sex, education, occupation, marital status, and native country.
In the experiments, we use our algorithm to develop a dynamic differential private logistic regression. The logistic regression, i.e., $\mathcal{L}_{LR}$ takes the following form: $\mathcal{L}_{LR}(y_{ip}f^T x_{ip})= log(1+exp(-y_{ip} f_p^T x_{ip}))$,
whose first-order derivative and the second-order derivative can be bounded as $|\mathcal{L}'_{LR}|\leq 1$ and $|\mathcal{L}''_{LR}|\leq \frac{1}{4}$, respectively, satisfying Assumption 3. Therefore, the loss function of logistic regression satisfies the conditions shown in Assumption 2 and 3. 
In this experiment, we set $R(f_p) = \frac{1}{2}\parallel f_p \parallel^2$, and $c_1 = \frac{1}{4}$. We can directly apply the loss function $\mathcal{L}_{LR}$ to Theorem 1 and 2 with $R(f)=\frac{1}{2}\parallel f_p\parallel^2$, and $c_1 = \frac{1}{4}$, and then it can provide $\alpha_p(t)$-differential privacy for all $t\in\mathbb{Z}$.

We also study the privacy-accuracy tradeoff of Algorithm 2 and 3. The privacy is quantified by the value of $\alpha_p(t)$. A larger $\alpha_p(t)$ implies that the ratio of the densities of the classifier $f_p(t)$ on two different data sets is larger, which implies a higher belief of the adversary when one data point in dataset $D$ is changed; thus, it provides lower privacy. However, the accuracy of the algorithm increases as $\alpha_p(t)$ becomes larger. As shown in Figure 3, a larger $\alpha_p(t)$ leads to faster convergence of the algorithms; moreover, from Figure 3, we can see that the DVP is slightly more robust to noise than is the primal case given the same value of $\alpha_p(t)$.
When $\alpha_p(t)$ is small, the model is more private but less accurate. Therefore, the utilities of privacy and accuracy need to satisfy the following assumptions:
\begin{asu} 
The utilities of privacy is monotonically increasing with respect to $\alpha_p(t)$ for every $p\in\mathcal{P}$ but accuracy is monotonically decreasing with respect to $\alpha_p(t)$ for every $p\in\mathcal{P}$.
\end{asu}

The quality of classifier is measured by the total empirical loss $\overline{C}(t)=\frac{C^R}{B_p}\sum_{i=1}^{B_{p}}\mathcal{L}(y_{ip} f_{p}(t)^T x_{ip})$. Let $L_{acc}(\cdot):\mathbb{R}_+\rightarrow\mathbb{R}$ represent the relationship between $\alpha_p(t)$ and $\overline{C}(t)$. The function $L_{acc}$ is obtained by curve fitting given the experimental data points $(\alpha_p(t), \overline{C}(t))$.
Let $U_{priv}(\alpha_p(t)): \mathbb{R}_+\rightarrow \mathbb{R}$ be the utility of privacy, same for every node $p\in\mathcal{P}$. Besides the decreasing monotonicity, $U_{priv}(\alpha_p(t))$ is assumed to be convex and doubly differentiable function of $\alpha_p(t)$. In our experiment, we model the utility of privacy as:
%\begin{equation}\label{equiUprivacy}
$U_{priv}(\alpha_p(t)) = \omega_{p1}\cdot\ln\frac{\omega_{p2}}{\omega_{p3}\alpha_p(t)+\omega_{p4}\alpha^2_p(t)}$,
where, $\omega_{pj}\in\mathbb{R}$ for $j=1,\:2,\:3,\:4$.
%\end{equation}
% 
For training the classifier, we use a few fixed values of $\rho$ and test the empirical loss $\overline{C}(t)=\frac{C^R}{B_p}\sum_{i=1}^{B_p}\mathcal{L}_{LR}(t)$ of the classifier. Then, we select the value of $\rho$ that minimizes the empirical loss for a fixed $\alpha_{p}$ ($0.3$ in this experiment). We also test the non-private version of algorithm, and the corresponding minimum $\rho$ is obtained as the control.
We choose the corresponding optimal values of the regularization parameter $\rho$ as $10^{-10}$, $10^{-2.5}$ and $10^{-1}$ for Algorithm 1, 2 and 3, respectively. The values of $C^R$ are chosen as $1750$, $1750$ and $146$ for  Algorithm 1, 2 and 3, respectively. 
Figure 5 shows the convergence of DVP and PVP at different values of $\alpha_p(t)$ at a given iteration $t$.
Larger values of $\alpha_p$ yield better convergence for both perturbations. Moreover, the DVP has a smaller variance of empirical loss than the primal perturbation does. However, a larger $\alpha_p$ leads to poorer privacy. 
Figure 4 $(a)$-$(b)$ shows the privacy-accuracy tradeoff of DVP at different iterations. By curve fitting, we model the function 
% 28
%\begin{equation}\label{equiUaccuracy}
$
L_{acc}(\alpha_p(t)) = c_4\cdot e^{-c_5\alpha_p(t)}+ c_6,
$
%\end{equation}
where $c_4$, $c_5, c_6 \in\mathbb{R}_+$. From the experimental results, we determine $c_4=0.2$, $c_5=25$, $c_6=\min_{t} \{\overline{C}(t)\}$; these values are applicable at all iteraions.
Figure 4 $(c)$-$(d)$ presents the privacy-accuracy tradeoff of PVP at different iterations. We model the function $L_{acc}$ in the same way as DVP. In our experiment, we choose $\omega_{p1}=0.02$,  $\omega_{p2}=6$, $\omega_{p3}=9$,  $\omega_{p4}=1$.
From Figure 3, we can see that the experimental results of $L_{acc}(\alpha_p(t))$ given $\{\alpha_p(t)\}$ for PVP experimences more oscillations than the DVP does. For iteration $t>1$,  $c_4=20$, $c_5=20$, $c_6=\frac{1}{81}\sum_{t=20}^{100} \overline{C}(t)$. As shown in Figure 3, the empirical loss of DVP is more robust to noise than the PVP for most values of $\alpha_p(t)$. Moreover, the dual perturbation yields a lower error rate for a large range of values of $\alpha_p(t)$, which implies a better management of tradeoff between privacy and accuracy. Figure 5 shows the privacy-accuracy tradeoff of the final optimum classifier in terms of the empirical loss and misclassification error rate (MER). The MER is determined by the fraction of times the trained classifier predicts a wrong label. We can see that PVP performs slightly better than DVP with respect to the empirical loss.

% error rate
% error rate
\begin{figure}%
\centering
\begin{subfigure}{.5\columnwidth}
\includegraphics[width=\columnwidth]{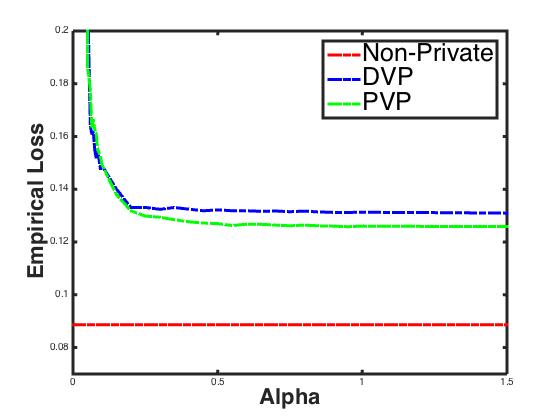}%
\caption{Empirical risk vs. $\alpha_p$.}%
\label{subfigb}%
\end{subfigure}\hfill%
\begin{subfigure}{.5\columnwidth}
\includegraphics[width=\columnwidth]{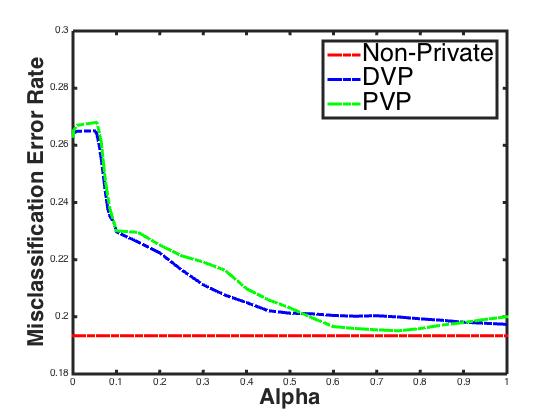}%
\caption{Misclassification error rate vs. $\alpha_p$.}%
\label{subfigc}%
\end{subfigure}\hfill%
\caption{Privacy-accuracy tradeoff. (a): Empirical risk vs. $\alpha_p$ of final optimum output. (b): Misclassifications error rate vs. $\alpha_p$ of iteration 100.}
\label{figabc}
\end{figure}

\section{Conclusion}
In this work, we have developed two ADMM-based algorithms to solve a centralized regularized ERM in a distributed fashion while providing $\alpha$-differential privacy for the ADMM iterations as well as the final trained output. Thus, the sensitive information stored in the training dataset at each node is protected against both the internal and the external adversaries.

Based on distributed training datasets, Algorithm 2 perturbs the dual variable $\lambda_p(t)$ for every node $p\in\mathcal{P}$ at iteration $t$; For the next iteration, $t+1$, the perturbed version of $\lambda_p(t)$ is involved in the update of primal variable $f_p(t+1)$. Thus, the perturbation created at time $t$ provides privacy at time $t+1$. In Algorithm 3, we perturb the primal variable $f_p(t)$, whose noisy version is then released to the neighboring nodes. Since the primal variables are shared among all the neighboring nodes, at time $t$, the noise directly involved in the optimization of parameter update comes from multiple nodes; as a result, the updated primal variable has more randomness than the dual perturbation case. 

In general, the accuracy decreases as privacy requirements are more stringent. The tradeoff between the privacy and accuracy is studied. 
Our experiments on real data from UCI Machine Learning Repository  show that dual variable perturbation is more robust to the noise than the primal variable perturbation. The dual variable perturbation outperforms the primal case at balancing the privacy-accuracy tradeoff as well as learning performance.

\appendices

%%%%%%%% 
\section{Proof of Theorem 1}
\begin{proof}(\textbf{Theorem 1})

Let $f_p(t+1)$ be the optimal primal variable with zero duality gap. 
From the Assumption 1 and 2, we know that both the loss funciton $\mathcal{L}$ and the regularizer $R(\cdot)$ are differentiable and convex, and by using the Karush-Kuhn-Tucker (KKT) optimality condition (stationarity), we have the relationship between the noise $\epsilon_p(t)$ and the optimal primal variable $f_p(t+1)$ as:
%$$
%\begin{aligned}
%0=& \frac{C^R}{B_p}\sum_{i=1}^{B_{p}}y_{ip}\mathcal{L}'(y_{ip} f_{p}(t+1)^T x_{ip})x_{ip}+\rho\nabla R(f_p)\\
%&+2\big(\frac{C^R}{2B_p}\epsilon_p(t)+\lambda_{p}(t)\big)+(\Phi+2\eta N_{p})f_{p}(t+1)\\
%&-\eta\sum_{i\in \mathcal{N}_p}(f_p(t)+f_i(t)),
%\end{aligned} 
%$$
%from which we can establish the relationship between the noise $\epsilon_p(t)$ and the optimal primal variable $f_p(t+1)$ as:
\begin{equation}\label{appendixNFp}
\begin{aligned}
\epsilon_p(t)=& -\sum_{i=1}^{B_{p}}y_{ip}\mathcal{L}'(y_{ip} f_{p}(t+1)^T x_{ip})x_{ip}-\frac{B_p}{C^R}\rho\nabla R(f_p)\\
&-\frac{2B_p}{C^R}\lambda_{p}(t)-\frac{B_p}{C^R}(\Phi+2\eta N_{p})f_{p}(t+1)\\
&+\frac{B_p\eta}{C^R}\sum_{i\in \mathcal{N}_p}(f_p(t)+f_i(t)).
\end{aligned}
\end{equation}
Under Assumption 1, the augmented Lagrange function $L^{dual}_p(t)$ is strictly convex, thus there is a unique value of $f_p(t+1)$ for fixed $\epsilon_p(t)$ and dataset $D_p$. The equation (\ref{appendixNFp}) shows that for any value of $f_p(t+1)$, we can find a unique value of $\epsilon_p(t)$ such that $f_p(t+1)$ is the minimizer of $L_p^{dual}$. Therefore, given a dataset $D_p$, the relation between $\epsilon_p(t)$ and $f_p(t+1)$ is bijective.

Let $D_p$ and $D'_p$ be two datasets with $H_d(D_p,D'_p)=1$, $(x_i, y_i)\in D_p$ and $(x'_i,y'_i)\in D'_p$ are the corresponding two different data points. Let two matrices $\textbf{J}_{f}(\epsilon_{p}(t)|D_p)$ and $\textbf{J}_{f}(\epsilon'_p(t)|D'_p)$ denote the Jacobian matrices of mapping from $f_p(t+1)$ to $\epsilon_p(t)$ and $\epsilon'_p(t)$, respectively. 
Then, transformation from noise $f_p(t+1)$ to $\epsilon_p(t)$ by Jacobian yields: 
\begin{equation}
\begin{aligned}
\frac{Q(f_p(t+1)|D_{p})}{Q(f_p(t+1)|D'_p)}=\frac{q (\epsilon_p(t)|D_{p})}{q (\epsilon'_p(t)|D'_p)}\frac{|\det(\textbf{J}_{f}(\epsilon_p(t)|D_p))|^{-1}}{|\det(\textbf{J}_{f}(\epsilon'_p(t)|D'_p))|^{-1}},
\end{aligned}
\end{equation}
where $q (\epsilon_{p}(t)|D_{p})$ and $q (\epsilon'_p(t)|D'_p)$ are the densities of $\epsilon_p(t)$ and $\epsilon'_p(t)$, respectively, given $f_p(t+1)$ when the datasets are $D_{p}$ and $D'_p$, respectively. 

Therefore, in order to prove the ratio of conditional densities of optimal primal variable is bounded as:
$
\frac{Q(f_p(t)|D)}{Q(f_p(t)|D'_p)}\leq e^{\alpha_p(t)},
$
we have to show:
$
\begin{aligned}
\frac{q (\epsilon_p(t)|D_{p})}{q (\epsilon'_p(t)|D'_p)}&\cdot\frac{|\det(\textbf{J}_{f}(\epsilon_p(t)|D_p))|^{-1}}{|\det(\textbf{J}_{f}(\epsilon'_p(t)|D'_p))|^{-1}}\leq e^{\alpha_p(t)}.
\end{aligned}
$
We first bound the ratio of the determinant of Jacobian matrices, and then the ratio of conditional densities of the noise vectors.

Let $x^{a}$ be the $a$-th element of the vector $x$, and $(a,b)$. Let $\textbf{E} \in \mathbbm{R}^{d\times d}$ be a matrix, then let $\textbf{E}^{(a,b)}$ denote the $(a,b)$-th entry of the matrix $\textbf{E}$. Thus, the $(m, n)$-th entry of $\textbf{J}_{f}(\epsilon_p(t))$ is:
{\small $$
\begin{aligned}
\textbf{J}_{f}(\epsilon_p(t))^{(m,n)}=& -\sum_{i=1}^{B_{p}}(y_{i}^2\mathcal{L}''(y_{i}f_p(t+1)^Tx_{i})x_{i}^{(m)}x_{i}^{(n)}\\
& -\frac{B_p}{C^R}\rho\nabla^2 R(f_p(t+1))^{(m,n)}\\
&-\frac{B_p}{C^R}(\Phi+2\eta N_{p})\mathbbm{1}(j=k).\\
\end{aligned}
$$}
Let $\textbf{J}_f^{0}(x_i,y_i) = (y_{i}^2\mathcal{L}''(y_{i}f_p(t+1)^Tx_{i})x_{i}x_{i}^T$, then the Jacobian matrix can be expressed as:
$$
\begin{aligned}
\textbf{J}_f(\epsilon_p(t)|D_p) =& -\sum_{i=1}^{B_{p}}\textbf{J}_f^{0}(x_i,y_i)-\frac{B_p}{C^R}\rho\nabla^2 R(f_p(t+1))\\
&-\frac{B_p}{C^R}(\Phi+2\eta N_{p})\textbf{I}_d.\\
\end{aligned}
$$

Let $\textbf{M}=\textbf{J}_f^{0}(x'_i,y'_i)-\textbf{J}_f^{0}(x_i,y_i)$, and $\textbf{H}=-\textbf{J}_f(\epsilon_p(t)|D_p)$, and thus $ \textbf{J}_f(\epsilon_p(t)|D'_p)=-(\textbf{M}+\textbf{H})$.
Let $h_j(\textbf{W})$ be the $j$-th largest eigenvalue of a symmetric matrix $\textbf{W}\in\mathbbm{R}^{d\times d}$ with rank $\theta$. Then, we have the following fact:
$
\det(\textbf{I}+\textbf{W})=\prod_j^{\theta} (1+h_j(\textbf{W})).
$
Since the matrix $x_ix_i^T$ has rank 1, matrix $\textbf{M}$ has rank at most 2; thus matrix $\textbf{H}^{-1}\textbf{M}$ has rank at most 2; therefore, we have:
$$
\begin{aligned}
\det(\textbf{H}+\textbf{M})&=\det(\textbf{H})\cdot \det(\textbf{I}+\textbf{H}^{-1}\textbf{M})\\
&=\det(\textbf{H})\cdot(1+h_1(\textbf{H}^{-1}\textbf{M}))(1+h_2(\textbf{H}^{-1}\textbf{M}).
\end{aligned}
$$
Thus, the ratio of determinants of the Jacobian matrices can be expressed as:
{\small $$ 
\begin{aligned}
\frac{|\det(\textbf{J}_{f}(\epsilon_p(t)|D_p))|^{-1}}{|\det(\textbf{J}_{f}(\epsilon'_p(t)|D'_p))|^{-1}}=&\frac{|\det(\textbf{H}+\textbf{M})|}{|\det(\textbf{H})|}\\
=&|\det(\textbf{I}+\textbf{H}^{-1}\textbf{M})|\\
=&(1+h_1(\textbf{H}^{-1}\textbf{M}))(1+h_2(\textbf{H}^{-1}\textbf{M})\\
=&|1+h_1(\textbf{H}^{-1}\textbf{M})+h_2(\textbf{H}^{-1}\textbf{M})\\
&+h_1(\textbf{H}^{-1}\textbf{M})h_2(\textbf{H}^{-1}\textbf{M})|.
\end{aligned}
$$ }

Based on Assumption 2, all the eigenvalues of $\nabla^2 R(f_p(t+1))$ is greater than 1 \cite{narayanan2008robust}. Thus, from Assumption 1, matrix $\textbf{H}$ has all eigenvalues at least $\frac{B_p}{C^R}\big(\rho+\Phi+2\eta N_{p}\big)$. Therefore, $|h_1(\textbf{H}^{-1}\textbf{M})|\leq \frac{|h_i(\textbf{M} )|}{\frac{B_p}{C^R}\big(\rho+\Phi+2\eta N_{p}\big)}$.
Let $\sigma_{i}(\textbf{M})$ be the non-negative singular value of the symmetric matrix $\textbf{M}$. According to \cite{chilstrom2013singular}, we have the inequality 
$
\sum_{i}|h_{i}(\textbf{M})|\leq \sum_{i}\sigma_{i}(\textbf{M}).
$
Thus, we have
$
|h_{1}(\textbf{M})|+|h_{2}(\textbf{M})|\leq \sigma_{1}(\textbf{M})+\sigma_{2}(\textbf{M}).
$
Let $\parallel \textit{X} \parallel_{\Sigma} = \sum_{i}\sigma_{i}$ be the trace norm of $\textit{X}$. Then, according to the \textit{trace norm inequality}, we have 
$
\parallel \textbf{M} \parallel_{\Sigma} \leq \parallel \textbf{J}^0(x'_i,y'_i)  \parallel_{\Sigma}+\parallel -\textbf{J}^0(x_i,y_i) \parallel_{\Sigma}.
$
As a result, based on the upper bounds from Assumption 1 and 3, we have:
$$
\begin{aligned}
|h_{1}(\textbf{M})|+|h_{2}(\textbf{M})|&\leq \parallel \textbf{J}^0(x'_i,y'_i)  \parallel_{\Sigma}+\parallel -\textbf{J}^0(x_i,y_i) \parallel_{\Sigma}\\
&\leq |(y_{i}^2\mathcal{L}''(y_{i}f_p(t+1)^Tx_{i})|\cdot \parallel x_{i}\parallel\\
&+|(y_{i}^{'2}\mathcal{L}''(y'_{i}f_p(t+1)^Tx'_{i})|\cdot \parallel x'_{i}\parallel\leq 2c_1,
\end{aligned}
$$
which follows $h_1(\textbf{M})h_2(\textbf{M})\leq c_1^2$. Finally, the ratio of determinants of Jacobian matrices is bounded as:
\begin{equation}
\begin{aligned}
\frac{|\det(\textbf{J}_{f}(\epsilon_p(t)|D_p))|^{-1}}{|\det(\textbf{J}_{f}(\epsilon'_p(t)|D'_p))|^{-1}} &\leq (1+\frac{c_1}{ \frac{B_p}{C^R}\big(\rho+\Phi+2\eta N_{p}\big))  })^2= e^{\overline{\alpha}},
\end{aligned}
\end{equation}
where $\overline{\alpha} = \ln\Big(1+\frac{c_1}{ \frac{B_p}{C^R}\big(\rho+\Phi+2\eta N_{p}\big)}\Big)^2$.

Now, we bound the ratio of densities of $\epsilon_p(t)$. Let $sur(E)$ be the surface area of the sphere in $d$ dimension with radius $E$, and $sur(E)=sur(1)\cdot E^{d-1}$. We can write:
\begin{equation}
\begin{aligned}
\frac{q (\epsilon_{p}(t)|D_{p})}{q (\epsilon'_{p}|D'_p)} &= \frac{\mathcal{K}(\epsilon_p(t))\frac{\parallel \epsilon_{p}(t) \parallel^{d-1}}{sur(\parallel \epsilon_{1}(t) \parallel)}}{\mathcal{K}(\epsilon'_p(t))\frac{\parallel \epsilon'_{p}(t) \parallel^{d-1}}{sur(\parallel \epsilon'_{p}(t) \parallel)}} \leq e^{\zeta_p(t)(\parallel \epsilon'_{p}(t) \parallel - \parallel \epsilon_{p}(t) \parallel)}\leq e^{\hat{\alpha}_p },
\end{aligned}
\end{equation}
where $\hat{\alpha}_p$ is a constant satisfying the above inequality.
Since we want to bound the ratio of densities of $f_p(t+1)$ as 
$\frac{Q(f_p(t+1)|D_{p})}{Q(f_p(t+1)|D'_p)} \leq e^{\alpha_p(t)},$
we need
$
\hat{\alpha}_p  \leq \alpha_p(t) - \overline{\alpha}.
$
For non-negative $\Phi$, let 
$
\hat{\alpha}_p  = \alpha_p(t) - \ln\Big(1+\frac{c_1}{ \frac{B_p}{C^R}\big(\rho+2\eta N_{p}\big)}\Big)^2.
$
If $\hat{\alpha}_p >0$, then we fix $\Phi = 0$, and thus $\hat{\alpha}_p  = \alpha_p(t) - \overline{\alpha}$. Otherwise, let $\Phi= \frac{c_1}{\frac{B_p}{C^R}(e^{\alpha_p(t)/4}-1)}-\rho-2\eta N_p$, and $\hat{\alpha}_p  = \frac{\alpha_p(t)}{2}$, then $\hat{\alpha}_p  = \alpha_p(t) - \overline{\alpha}$. Therefore, we can have 
$\frac{|\det(J_{f}(b_1|D_p))|^{-1}}{|\det(J_{f}(b_2|D'_p))|^{-1}} \leq e^{\alpha_p(t) - \hat{\alpha}_p }.$
From the upper bounds stated in Assumption 1 and 3, the $l_2$ norm of the difference of $\epsilon_{1}$ and $\epsilon_{2}$ can be bounded as:
$$
\begin{aligned}
\parallel \epsilon'_{p}(t)-\epsilon_{p}(t) \parallel =& \sum_{i=1}^{B_{p}}\parallel y_{ip}\mathcal{L}'(y'_{ip} f_{p}(t+1)^T x'_{ip})x'_{ip}\\
&-(y_{ip}\mathcal{L}'(y_{ip} f_{p}(t+1)^T x_{ip})x_{ip}\parallel \leq 2.
\end{aligned}
$$
Thus,$ \parallel \epsilon'_{p}(t) \parallel - \parallel \epsilon_{p}(t) \parallel \leq \parallel \epsilon'_{p}(t)-\epsilon_{p}(t) \parallel \leq 2$. Therefore, by selecting $\zeta_p(t) = \frac{\hat{\alpha}_p}{2} $, we can bound the ratio of conditional densities of $f_p(t+1)$ as
$\frac{Q(f_p(t+1)|D_{p})}{Q(f_p(t+1)|D'_p)} \leq e^{\alpha_p(t)},$
and prove that the DVP can provide $\alpha_p(t)$-differential privacy. 
\end{proof} 

%\section{Proof of Corollary 1.2}
%\begin{proof}(\textbf{Corollary 1.2})
%We prove this corollary by induction. For $c_2=1$, it is true since this is exactly the case of Theorem 1. Suppose Corollary 1.2 is held for $H_d(D_p,D'_p)=c_2$. Let $H_d(D_p,D'_p)=c_2+1$. Clearly, there must exist a dataset $D''_p$ such that $H_d(D_p, D''_p) = 1$, and $H_d(D'_p, D''_p) = c_2$. Thus, from (13), we have:
%\begin{equation}
%\begin{aligned}
%\frac{Q(f_p(t)|D_p)}{Q(f_p(t)|D'_p)} &= \frac{Q(f_p(t)|D_p)}{Q(f_p(t)|D''_p)}\cdot\frac{Q(f_p(t)|D''_p)}{Q(f_p(t)|D'_p)} \\
%&\leq e^{\alpha_p(t)}e^{c_2\alpha_p(t)}=e^{(c_2+1)\alpha_p(t)}.
%\end{aligned}
%\end{equation}
%Therefore, the induction hypothesis is true and Corollary 1.2 is proven.
%\end{proof}

\section{Proof of Theorem 2}

\begin{proof}(\textbf{Theorem 2})

Let $D_p$ and $D'_p$ be two datasets with $H_d(D_p,D'_p)=1$. Since only $V_p(t)$ is released, then our target is to prove 
$
\frac{Q(V_p(t+1)|D_{p})}{Q(V_p(t+1)|D'_p)}\leq e^{\alpha_p(t)}.
$
From (\ref{equiVp}), we have:
$
\frac{Q(V_p(t+1)|D_{p})}{Q(V_p(t+1)|D'_p)}=\frac{\mathcal{K}(\epsilon_p(t))}{\mathcal{K}(\epsilon'_p(t))} = \frac{e^{-\zeta_p(t) \parallel \epsilon_p(t) \parallel}}{e^{-\zeta_p(t) \parallel \epsilon'_p(t) \parallel}}.
$
Therefore, in order to make the model to provide $\alpha_{p}(t)$-differential privacy, we need to find a $\zeta_p(t)$ that satisfies
\begin{equation}\label{appendixTheta}
\zeta_p(t)(\parallel \epsilon_p(t) \parallel - \parallel \epsilon'_p(t) \parallel)\leq \alpha_{p}(t).
\end{equation}
Let $V^A=\arg\min_{V_p}L_p^{prim}(t|D_p)$, and $V^B=\arg\min_{V_p}L_p^{prim}(t|D'_p)$, where $L_{prim}(t|D)$ is the augmented Lagrange function for PVP given dataset $D$.

Let $F$, $G$ be defined at each node $p\in\mathcal{P}$ as:
$F(V_p(t)) = L_p^{prim}(t|D_p),$ and 
$G(V_p(t)) = L_p^{prim}(t|D'_p) - L_p^{prim}(t|D_p)$, respectively.
Thus, $G(V_p) = \frac{C^R}{B_p}\sum_{i=1}^{B_{p}}(\mathcal{L}(y'_{ip} V_{p}^T x'_{ip})-\mathcal{L}(y_{ip}V_{p}^T x_{ip})).$
According to Assumption 2, we can imply that $L_p^{prim}(t|D_p)=F(V_p(t))$ and $L_p^{prim}(t|D'_p)=F(V_p(t))+G(V_p(t))$ are both $\rho$-strong convex. 
Differentating $G(V_p(t))$ with respect to $V_p(t)$ gives:
$
\begin{aligned}
\nabla G(V_p) =\frac{C^R}{B_p}(y'_{ip}\mathcal{L}'(y'_{ip} V_{p}^T x'_{ip})x'_{ip}
-(y_{ip}\mathcal{L}(y_{ip}V_{p}^T x_{ip})x_{ip} .
\end{aligned}
$
From Assumption 1 and 3, $\parallel\nabla G(V_p)\parallel\leq \frac{2C^R}{B_p}$. From definitions of $V^A$ and $V^B$, we have:
$\nabla F(V^A)=\nabla F(V^B)+\nabla F(V^B)=0$.
From Lemma 14 in \cite{shalev2007online} and the fact that $F(\cdot)$ is $\rho$-strongly convex, weh have the following inequality:
$
\langle\nabla F(V^A)-F(V^B), V^A-V^B\rangle \geq \rho \parallel V^A-V^B \parallel^2;
$
therefore, Cauchy-Schwarz inequality yields:
$$
\begin{aligned}
&\parallel V^A-V^B \parallel\cdot \parallel \nabla G(V^B) \parallel 
\geq (V^A-V^B)^T\nabla G(V^B)\\
& = \langle\nabla F(V^A)-F(V^B), V^A-V^B\rangle
\geq \rho \parallel V^A-V^B \parallel^2.
\end{aligned}
$$
Dividing both sides by $\rho \parallel V^A-V^B \parallel$ gives:
\begin{equation}\label{appendixV_V}
\parallel V^A-V^B \parallel\leq \frac{1}{\rho}\parallel \nabla G(V^B) \parallel \leq \frac{2C^R}{\rho B_p}.
\end{equation}
From (\ref{equiVp}), we have $ \parallel V^A-V^B \parallel\leq \frac{1}{\rho}\parallel \nabla G(V^B) \parallel = \parallel \epsilon_p(t) - \epsilon'_p(t) \parallel .$
Thus, we can bound 
$$
\begin{aligned}
\zeta_p(t)(\parallel \epsilon_p(t) \parallel - \parallel \epsilon'_p(t) \parallel)&\leq \zeta_p(t) (\parallel \epsilon_p(t) -\epsilon'_p(t) \parallel)\leq\frac{2C^R}{B_p\rho}\zeta_p(t)
\end{aligned}
$$
Therefore, by choosing $\zeta_p(t) = \frac{\rho B_p\alpha_p(t)}{2C^R}$, the inequality (\ref{appendixTheta}) holds; thus PVP is dynamic $\alpha_p$-differential private at each node $p$.

\end{proof}

\section{Proof of Theorem 3} 
\begin{proof} \textbf{(Theorem 3)}

Let
$
\hat{f}_p(t+1) = \arg\min_{f_p} \hat{Z}_p(f_p, t),
$ and
$
f^* = \arg\min_{f_p} Z_p(f_p,t|D_p)
$. Let $f^p(t+1) $ be the actual estimated optimum obtained using Algorithm 1. We assume that $f^p(t+1) $ is very close to the actually so that $Z_p(f^p(t+1)|D_p)- Z_p(f^*|D_p) \rightarrow 0$.
For the non-private ERM,  \cite{shalev2007online} and \cite{shalev2008svm} show that for a specific reference classifier $f^0$ at time $t+1$ such that $\hat{C}(f^0) = \hat{C}^0$, we have:
{\small $$
\begin{aligned}
\hat{C}(f^p(t+1))=&\hat{C}^0
+\big(\hat{Z}_p(f^p(t+1),t) -\hat{Z}_p(\hat{f}_p(t+1),t)\big)\\
+& \big( \hat{Z}_p(\hat{f}_p(t+1),t) - \hat{Z}_p(f^0,t) \big)\\
+& \frac{\rho}{2}\parallel f^0\parallel^2 - \frac{\rho}{2}\parallel f^p(t+1) \parallel^2.
\end{aligned}
$$}
From Sridharan et al. \cite{sridharan2009fast}, we have, with probability at least $1-\delta$
{\small $$
\begin{aligned}
\hat{Z}_p(f^p,t) -\hat{Z}_p(\hat{f}_p(t+1),t) &\leq 2\big(Z_p(f^p(t+1)|D_p)- Z_p(f^*|D_p) \big)\\
&+  \mathcal{O}\Big( C^R\frac{\ln(\frac{1}{\delta})}{B_p\rho}\Big).
\end{aligned}
$$}
Since $Z_p(f^p(t+1)|D_p)- Z_p(f^*|D_p)\rightarrow 0$, then,
$
\begin{aligned}
\hat{Z}_p(f^p(t+1),t) -\hat{Z}_p(\hat{f}_p(t+1),t) \leq  \mathcal{O}\Big(C^R \frac{\ln(\frac{1}{\delta})}{B_p\rho}\Big).
\end{aligned}
$
If we choose $\rho\leq \frac{\alpha_{acc}}{\parallel f^0 \parallel^2}$, then,
$
\frac{\rho}{2}\parallel f^0 \parallel^2 - \frac{\rho}{2}\parallel f^p(t+1) \parallel^2\leq\frac{\alpha_{acc}}{2}.
$
Thus,
$
\begin{aligned}
\hat{C}(f^p(t+1))\leq&\hat{C}^* + \mathcal{O}\Big( C^R\frac{\ln(\frac{1}{\delta})}{B_p\rho}\Big)+\frac{\alpha_{acc}}{2}.
\end{aligned}
$
Therefore, we can find the value of $B_p$ by solving
$
\mathcal{O}\Big(C^R \frac{\ln(\frac{1}{\delta})}{B_p\rho}\Big)+\frac{\alpha_{acc}}{2} \leq \alpha_{acc}
$,
we obtain:
$
B_p>\beta_{non}\max\Bigg( \frac{C^R\parallel f^0 \parallel^2 \ln(\frac{1}{\delta}) }{\alpha_{acc}^2} \Bigg).
$

If we determine different reference classifier $f^0_p(t+1)$ at different time, then we need to find the maximum value across the time and among different value of $\parallel f^0\parallel$:
{\small $$
B_p>\beta_{non}\max\Bigg( \max_t\Big(\frac{C^R\parallel f^0\parallel^2 \ln(\frac{1}{\delta}) }{\alpha_{acc}^2}\Big)  \Bigg).
$$}
Let $f_p^{non}(t+1) = \arg\min_{f_p}L_p^N(t)$. Since
$
\hat{C}(f_p^{non}(t+1)) = \hat{C}(f^p(t+1))+\Delta^{non}(t),
$
then,
$
\hat{C}(f_p^{non}(t+1)) \leq \hat{C}^0(t+1)+\alpha_{acc}+\Delta^{non}(t),
$
with probability no less than $1-\delta$.
\end{proof}

\section{}

In this appendix, we provide the proof of Theorem 4 with  the help of Lemma \ref{Lemma7} and \ref{Lemma8}, which are also proved later.

\begin{proof}\textbf{(Theorem 4)} 
First we define $\hat{f}_p(t+1)$ and $f^*$ in the same way as in Appendix C. We also define  
$
f^*_p(t+1) = \arg\min_{f_p} Z_p^{dual}(f_p,t|D_p),
$
and $\hat{C}(f^0_p(t))=\hat{C}^0(t)$ at time $t$. 
We use the analysis of \cite{shalev2007online} and \cite{shalev2008svm} (also see the work of Chaudhuri  et al. in \cite{chaudhuri2011differentially}), and have the following: 
{\small \begin{equation}\label{appendixProof_Th4}
\begin{aligned}
\hat{C}(f^*_p(t+1)) =& \hat{C}^0(t+1)
+\big( \hat{Z}_p(f^*_p(t+1),t) -\hat{Z}_p(\hat{f}_p(t+1),t)  \big)\\
+& \big( \hat{Z}_p(\hat{f}_p(t+1),t) - \hat{Z}_p(f^0_p(t+1),t) \big)\\
+& \frac{\rho}{2}\parallel f^0_p(t+1) \parallel^2 - \frac{\rho}{2}\parallel f^*_p(t+1) \parallel^2.
\end{aligned}
\end{equation} }
Now we bound each terms in the right hand side of (\ref{appendixProof_Th4}) as follows.
From Assumption 1, we have $\mathcal{L}'\leq c_1$. By choosing $B_p> \frac{5c_1C^R\parallel f^0_p(t+1) \parallel^2}{\alpha_{acc}\alpha_p(t)}$, and $\rho> \frac{\alpha_{acc}}{2\parallel f^0_p(t+1) \parallel^2}$,
and since $\alpha_p(t)\leq 1$, we have:
{\small $$
\begin{aligned}
\hat{\alpha}_p  =& \alpha_p(t) - \ln\Big(1+\frac{c_1}{ \frac{B_p}{C^R}\big(\rho+2\eta N_{p}\big)}\Big)^2
>\alpha_p(t) - \ln(1+\frac{c_1C^R}{B_p\rho})^2 \\
>& \alpha_p(t) - \ln(1+\frac{2\alpha_p(t)}{5})^2
>\alpha_p(t)-\frac{4\alpha_p(t)}{5} = \frac{\alpha_p(t)}{5}.
\end{aligned}
$$}
Then, according to Algorithm 2, we choose the corresponding $\zeta_p(t) = \frac{\alpha_p(t)}{4}$ because $\hat{\alpha}_p >0$.
Let $\Lambda$ be the event 
$$
\begin{aligned}
\Lambda:=& \Big\{ Z_p(f^*_p(t+1)|D_p)\leq Z_p(f^*|D_p)+\frac{16d^2\Big( \ln(\frac{d}{\delta})\Big)^2}{\rho B_p^2\alpha_p(t)^2}\Big\}.
\end{aligned}
$$
Since $\hat{\alpha}_p>\frac{\alpha_p(t)}{2}>0$, and applying Lemma \ref{Lemma8} yields
$
\begin{aligned}
\mathbb{P}_{\epsilon_p(t)}\Big( \Lambda \Big)\geq 1-\delta.
\end{aligned}
$
From the work of Sridharan et al.  in \cite{sridharan2009fast}, the following inequality holds with probability $1-\delta$ 
{\small $$
\begin{aligned}
&\hat{Z}_p(f^*_p(t+1))-\hat{Z}_p(\hat{f}_p(t+1)) \\
&\leq  2\Big(Z_p(f^*_p(t+1)|D_p)
-Z_p(f^*|D_p)\Big)+ \mathcal{O}\Big( \frac{\ln(\frac{1}{\delta})}{B_p\rho}\Big)\\
&\leq\frac{32d^2\Big( \ln(\frac{d}{\delta})\Big)^2}{\rho B_p^2\alpha_p(t)^2}+\mathcal{O}\Big( \frac{\ln(\frac{1}{\delta})}{B_p\rho}\Big).
\end{aligned}
$$}
The big-$\mathcal{O}$ notation hides the numerical constants, which depend on the derivative of the loss function and the upper bounds of the data points shown in Assumption 3.
Combining the above two processes, $\hat{Z}_p(f^*_p(t+1))-\hat{Z}_p(\hat{f}_p(t+1)) $ is bounded as shown above with probability $1-2\delta$.

From the definitions of $f^0_p(t+1)$ and $\hat{f}_p(t+1)$, we obtain $\hat{Z}_p(\hat{f}_p(t+1),t)-\hat{Z}_p(f^0_p(t+1),t)<0$. Since $P\geq 1$, then by selecting $\rho = \frac{\alpha_{acc}}{\parallel f^0_p(t+1) \parallel^2}$, we can bound 
$
\frac{\rho}{2}\parallel f^0_p(t+1) \parallel^2 - \frac{\rho}{2}\parallel f^*_p(t+1) \parallel^2 \leq \frac{\alpha_{acc}}{2}.
$
Therefore, from (47), we have:
{\small $$
\begin{aligned}
\hat{C}(f^*_p(t+1)) \leq & C^0_{E} + \frac{32d^2\Big( \ln(\frac{d}{\delta})\Big)^2}{\rho B_p^2\alpha_p(t)^2}+ \mathcal{O}\Big( C^R\frac{\ln(\frac{1}{\delta})}{B_p\rho}\Big) + \frac{\alpha_{acc}}{2},
\end{aligned}
$$}
with $\rho = \frac{6\alpha_{acc}}{\parallel f^0_p(t+1) \parallel^2}$.
The lower bounds of $B_p$ is determined by solving the following:
$$
\begin{aligned}
\frac{32d^2\Big( \ln(\frac{d}{\delta})\Big)^2}{\rho B_p^2\alpha_p(t)^2}+ \mathcal{O}\Big(C^R \frac{\ln(\frac{1}{\delta})}{B_p\rho}\Big)+\frac{\alpha_{acc}}{2} \leq \alpha_{acc}.
\end{aligned}
$$

\end{proof}
\begin{lemma}\label{Lemma7}
\textit{Let $Z$ be a gamma random variable with density function $\Gamma(k,\theta)$, where $k$ is an integer, and let $\delta >0$. Then, we have:
$$
\mathbbm{P}(Z<k\theta\ln(\frac{k}{\delta}))\geq 1-\delta.
$$}
\end{lemma}

\begin{proof}\textbf{(Lemma \ref{Lemma7})} Since $Z$ is a gamma random variable $\Gamma(k,\theta)$, then we can express $Z$ as 
$
Z = \sum_{i=1}^kZ_i,
$
where $\{Z_i\}_{i=1}^k$ are independent exponential random variables with density function \textit{Exp}$(\frac{1}{\theta})$; thus, for each $Z_i$ we have:
$
\mathbbm{P}(Z_i\leq \theta \ln(\frac{k}{\delta})) = 1-\frac{\delta}{k}.
$
Since $\{Z_i\}_{i=1}^k$ are independent, we have:
$$
\begin{aligned}
\mathbbm{P}(Z<k\theta\ln(\frac{k}{\delta}))=\prod_{i=1}^k\mathbbm{P}(Z_i\leq \theta ln(\frac{k}{\delta}))
=(1-\frac{\delta}{k})^k \geq 1-\delta.
\end{aligned}
$$

\end{proof}

\begin{lemma} \label{Lemma8}
\textit{Let $\hat{\alpha}_p>0$, and $f^*_p(t+1)=\arg\min_{f_p} Z^{dual}_p(f_p,t|D_p)$, and $f^*=\arg\min_{f_p} Z_{p}(f_p|D_p)$. Let $\Lambda$ be the event 
$$
\begin{aligned} 
\Lambda:=\Big\{ Z_{p}(f^*_p(t+1)|D_p)\leq Z_{p}(f^*|D_p) 
+\frac{16d^2\Big( \ln(\frac{d}{\delta})\Big)^2}{\rho B_p^2\alpha_p(t)^2}\Big\}.
\end{aligned}
$$
Under Assumption 1 and 2, we have:
{\small $
\begin{aligned}
\mathbb{P}_{\epsilon_p(t)}\Big( \Lambda \Big)\geq 1-\delta.
\end{aligned}
$}
The probability $\mathbb{P}_{\epsilon_p(t)}$ is taken over the noise vector $\epsilon_p(t)$.
}
\end{lemma}

\begin{proof} \textbf{(Lemma \ref{Lemma8})}
Since $\hat{\alpha}_p>0$, $\Phi = 0$, $f^*_p(t+1)=\arg\min_{f_p} Z^{dual}_p(f_p,t|D_p)$ can be expressed as:
$
f^*_p(t+1) = \arg\min_{f_p} \Big( Z_p(f_p|D_p) + 2\epsilon_p(t)^Tf_p \Big).
$
Thus, we have:
$$
\begin{aligned}
Z_p(f^*_p(t+1)|D_p)\leq Z_p(f^*|D_p)
+\frac{C^R}{B_p}\epsilon_p(t)^T(f^*- f^*_p(t+1)).
\end{aligned}
$$
Firstly, we bound the $l_2$-norm $\parallel f^*- f^*_p(t+1) \parallel$. We use the similar procedure to establish (\ref{appendixV_V}) in Appendix C by setting $F(Y) = Z_{p}(Y|D_p)$ and $G(Y) = \frac{C^R}{B_p}\epsilon_p(t)$; thus, based on Assumption 1 and 2, we have:
$$
\begin{aligned}
\parallel f^*- f^*_p(t+1) \parallel\leq \frac{1}{\rho}\parallel \nabla\big(2\epsilon_p(t)^Tf_p\big)\parallel
\leq\frac{C^R\parallel \epsilon_p(t) \parallel}{B_p\rho}.
\end{aligned}
$$
Cauchy-Schwarz inequality yields:
{\small $$
\begin{aligned}
&Z_{p}(f^*_p(t+1)|D_p) -Z_{p}(f^*|D_p)\leq \parallel Z_{p}(f^*_p(t+1)|D_p) -Z_{p}(f^*|D_p) \parallel \\
&\leq \frac{2}{B_p}\parallel\epsilon_p(t)^T(f^*- f^*_p(t+1) \parallel\leq \frac{\big(C^R\big)^2\parallel \epsilon_p(t) \parallel^2}{B_p^2\rho}.
\end{aligned}
$$}
Since the noise vector $\epsilon_p(t)$ is drawn from 
$
\mathcal{K}_p(\epsilon) \sim e^{-\zeta_p(t) \parallel \epsilon \parallel},
$
then $\parallel \epsilon_p(t) \parallel$ is drawn from $\Gamma(d,\frac{1}{\zeta_p(t)})=\Gamma(d,\frac{2}{\hat{\alpha}_p})$. Then, by using Lemma \ref{Lemma7} with $\parallel \epsilon_p(t) \parallel\leq \frac{2d\ln(\frac{d}{\delta})}{\hat{\alpha}_p}$, we have:
{\small $$
\begin{aligned}
\Lambda:=\Big\{ Z_{p}(f^*_p(t+1)|D_p)\leq Z_{p}(f^*|D_p)-\frac{4d^2\Big( \ln(\frac{d}{\delta})\Big)^2}{\rho B_p^2\alpha_p(t)^2}.
\end{aligned}
$$}
with probability no less than $1-\delta$.
\end{proof}

\section{}
We prove Theorem 5 here. This appendix also shows Lemma\ref{Lemma11} that are used in the proof of Theorem 5.
\begin{proof} \textbf{(Theorem 5)}
We define $\hat{f}_p(t+1)$ and $f^*$ as in the proof of Theorem 4 in Appendix D, and we also define 
$
f^*_p(t+1) = \arg\min_{f_p} Z_p^{prim}(f_p,t|D_p).
$
Let $\hat{C}(f^0_p(t))=\hat{C}^0(t)$ at time $t$. 
As in Appendix D, we again use the analysis of \cite{shalev2007online} and \cite{shalev2008svm} (also see the work of Chaudhuri  et al. in \cite{chaudhuri2011differentially}), and have the follows 
{\small \begin{equation}\label{appendixProof_Th5}
\begin{aligned}
\hat{C}(f^*_p(t+1)) =& \hat{C}(f^0_p(t+1)) 
+\big( \hat{Z}_p(f^*_p(t+1),t) \\
-&\hat{Z}_p(\hat{f}_p(t+1),t)  \big)\\
+&\big( \hat{Z}_p(\hat{f}_p(t+1),t) - \hat{Z}_p(f^0_p(t+1),t) \big)\\
+& \frac{\rho}{2}\parallel f^0_p(t+1) \parallel^2 - \frac{\rho}{2}\parallel f^*_p(t+1) \parallel^2.
\end{aligned}
\end{equation}}

According to Theorem 2, we choose $\zeta_p(t)=\frac{\rho B_p\alpha_p(t)}{2C^R}> 0$. Thus, applying Lemma \ref{Lemma11}, we have:
$$
\begin{aligned}
Z_p(f^*_p(t+1)|D_p) -Z_p(f^*|D_p)\leq\frac{16\big(C^R\big)^2\eta^2N_p^2 d^2 \big(\ln(\frac{d}{\delta}) \big)^2}{\rho^3 B_p^2 \alpha_p(t)^2},
\end{aligned}
$$
with probability no smaller than $1-\delta$. 
Then, we use the result of Sridharan et al. in \cite{sridharan2009fast}, with probability no smaller than $1-\delta$:
{\small $$
\begin{aligned}
&\hat{Z}_p(f^*_p(t+1))-\hat{Z}_p(\hat{f}_p(t+1)) \leq  2\Big(Z_p(f^*_p(t+1)|D_p)\\
-&Z_p(f_p^*(t+1)|D_p)\Big)+ \mathcal{O}\Big( \frac{\ln(\frac{d}{\delta})}{B_p\rho}\Big)\\
\leq& \frac{32\big(C^R\big)^2\eta^2N_p^2 d^2 \big(\ln(\frac{d}{\delta}) \big)^2}{\rho^3 B_p^2 \alpha_p(t)^2}+ \mathcal{O}\Big( \frac{\ln(\frac{1}{\delta})}{B_p\rho}\Big).
\end{aligned}
$$}
Combining the above two processes, we have the probability no smaller than $1-2\delta$.

In order to bound the last two terms in (\ref{appendixProof_Th5}), we select $\rho = \frac{\alpha_{acc}}{\parallel f^0_p(t+1) \parallel^2}$; as a result, 
$
\begin{aligned}
\frac{\rho}{2}\parallel f^0_p(t+1) \parallel^2 - \frac{\rho}{2}\parallel f^*_p(t+1) \parallel^2 \leq \frac{\alpha_{acc}}{2}.
\end{aligned}
$
From the definitions of $ \hat{f}_p(t+1)$ and $f^0_p(t+1) $, we have: $
\hat{Z}_p(\hat{f}_p(t+1),t) - \hat{Z}_p(f^0_p(t+1),t)\leq 0.
$
The value of $B_p$ is determined such that
$
\hat{C}(f^*_p(t+1)) \leq \hat{C}^* +\alpha_{acc}.
$
Therefore, we find the bounds of $B_p$ by solving
$$
\begin{aligned}
    \frac{32\big(C^R\big)^2\eta^2N_p^2 d^2 \big(\ln(\frac{d}{\delta}) \big)^2}{\rho^3 B_p^2 \alpha_p(t)^2}
+ \mathcal{O}\Big( \frac{C^R\ln(\frac{1}{\delta})}{B_p\rho}\Big) &+\frac{\alpha_{acc}}{2}\leq \alpha_{acc},
\end{aligned}
$$
with  $\rho = \frac{\alpha_{acc}}{\parallel f_p^0(t+1) \parallel^2}$.
\end{proof}

%%%%
The following Lemma is analogous to Lemma \ref{Lemma7}. 
\begin{lemma}\label{Lemma11}
\textit{Let $\zeta_p(t)>0$, and $f^*_p(t+1)=\arg\min_{f_p} Z_p^{prim}(f_p,t|D_p)$, and $f^*=\arg\min_{f_p} Z_{p}(f_p|D_p)$. Suppose that the noise vector $\epsilon_t(t)$ generated at time $t$ has the same value of $\alpha_p(t)$ for all $p\in \mathcal{P}$.
Let $\Lambda$ be the event 
$$
\begin{aligned}
\Lambda:=& \Big\{ Z_p(f^*_p(t+1)|D_p)\leq Z_p(f^*|D_p) \\
&+\frac{16\big(C^R\big)^2\eta^2N_p^2 d^2 \big(\ln(\frac{d}{\delta}) \big)^2}{\rho^3 B_p^2 \alpha_p(t)^2}\Big\}.
\end{aligned}
$$
If the loss function $\mathcal{L}$ is convex and differentiable with $|\mathcal{L}|\leq 1$, then, we have:
{\small $
\begin{aligned}
\mathbb{P}_{\epsilon_p(t)}\Big( \Lambda \Big)\geq 1-\delta.
\end{aligned}
$ }
The probability $\mathbb{P}_{\epsilon_p(t)}$ is taken over the noise vector $\epsilon_p(t)$.
}

\end{lemma}

\begin{proof}\textbf{(Lemma \ref{Lemma11})}

Let $\epsilon^{pi}(t)=\epsilon_p(t)-\epsilon_i(t)$ with probability density $P_{\epsilon^{pi}}$. 
Let $f^*_p(t+1)=\arg\min_{f_p} Z_p^{prim}(f_p,t|D_p)$, and it can be expressed as:
$
\begin{aligned}
&f^*_p(t+1) =\arg\min_{f_p}\Big( Z_p(f_p|D_p)- Y_p\Big),
\end{aligned}
$
where $Y_p = \eta \sum_{i\in \mathcal{N}_p}\Big((f_p -\frac{1}{2}(f_p(t)+f_i(t))^T\cdot(\epsilon^{pi}(t))+\frac{1}{4}\big(\epsilon^{pi}(t)\big)^2\Big)$.
Thus, we have:
$$
\begin{aligned}
Z_p(f^*_p(t+1)&|D_p)\leq Z_p(f^*|D_p)
-\eta\sum_{i\in \mathcal{N}_p}(f^*-f^*_p(t+1))^T\cdot\epsilon^{pi}.
\end{aligned}
$$

Firstly, we bound the $l_2$-norm $\parallel f^*- f^*_p(t+1) \parallel$. We use the similar procedure to establish (46) in Appendix D by setting $F(\cdot) = Z_p(\cdot|D_p)$ and $G(\cdot) = \eta \sum_{i\in \mathcal{N}_p}\big(\epsilon^{pi}\big)^T(\cdot)$; thus, based on Assumption 1 and 2, we have:
{\small $$
\begin{aligned}
&\parallel f^*- f^*_p(t+1) \parallel\leq \frac{1}{\rho}\parallel\sum_{i\in \mathcal{N}_p} \nabla(\eta N_p(f^*_p(t+1))^T\epsilon^{pi})\parallel\\
&\leq \sum_{i\in \mathcal{N}_p}\frac{\eta\parallel \epsilon^{pi}(t) \parallel}{\rho} =\sum_{i\in \mathcal{N}_p}\frac{\eta\Big(\parallel \epsilon_p(t) -\epsilon_j(t)\parallel\Big) }{\rho}\\
&\leq \sum_{i\in \mathcal{N}_p}\frac{\eta\Big(\parallel \epsilon_p(t) \parallel+\parallel \epsilon_j(t) \parallel\Big)}{\rho}.
\end{aligned}
$$ }
Since $\alpha_p(t)$ is the same for all $p\in \mathcal{P}$ at time $t$, $\zeta_j(t)=\frac{\rho B_p \alpha_p(t)}{2C^R}$ for all $j\in\mathcal{P}$. Since $\epsilon_j(t)$ is drawn from (15), then, $ \parallel \epsilon_p(t) \parallel \sim \Gamma(d,\frac{1}{\zeta_p(t)})$ for all $p\in\mathcal{P}$. Let 
$
\parallel \epsilon_{pi} \parallel^{\oplus} = \parallel \epsilon_p(t) \parallel+ \parallel\epsilon_i(t) \parallel.
$
Thus, 
{\small $$
\begin{aligned}
\parallel f^*- f^*_p(t+1) \parallel \leq \sum_{i\in\mathcal{N}_p}\frac{\eta\Big(\parallel \epsilon_{pi} \parallel^{\oplus}\Big)}{\rho}=\frac{\eta N_p\Big(\parallel \epsilon_{pi} \parallel^{\oplus}\Big)}{\rho} .
\end{aligned}
$$}
Cauchy-Schwarz inequality yields:
{\small $$
\begin{aligned}
&Z_p(f^*_p(t+1)|D_p)-Z_p(f^*|D_p) \\
&\leq \parallel Z_p(f^*_p(t+1)|D_p)-Z_p(f^*|D_p)  \parallel \leq \frac{\eta^2 N_p^2\Big(\parallel \epsilon_{pi} \parallel^{\oplus}\Big)^2}{\rho},
\end{aligned}
$$}
From the fact that if $\{X_j\}_{j=1}^{K}$ are independent gamma random variables with density $\Gamma(\beta_j, h)$, then $X = \sum_{j=1}^{K}X_j$ is a gamma random variable with $\Gamma(\sum_{j}^K\beta_j, h)$, we have $P_{ \parallel\epsilon^{pj} \parallel}=\Gamma(2d,\frac{2C^R}{\rho B_p\alpha_p(t)})$. Applying Lemma \ref{Lemma7} with $\parallel \epsilon^{pj}(t) \parallel^{\oplus}\leq \frac{4C^Rd\ln(\frac{d}{\delta})}{\rho B_p\alpha_p(t)}$ yields:
$$
\begin{aligned}
&Z_p(f^*_p(t+1)|D_p) -Z_p(f^*|D_p)\leq \frac{16\big(C^R\big)^2\eta^2N_p^2 d^2 \big(\ln(\frac{d}{\delta}) \big)^2}{\rho^3 B_p^2 \alpha_p(t)^2}
\end{aligned}
$$
with probability no smaller than $1-\delta$.
\end{proof}

\section{}
Theorem 6 is proved in this appendix based on Lemma \ref{Lemma12}.

\begin{proof} \textbf{(Theorem 6)}
We use 
$
\hat{f}_p(t+1),
$
$
f^*,
$
and 
$
f^*_p(t+1),
$
defined in the proof of Theorem 5 in Appendix E.
Now we use a reference $f^0_p(t)$ such that $\hat{C}(f^*_p(t)) = \hat{C}^0(t)$ be the reference at time $t+1$. 
We use the analysis of \cite{shalev2007online} and \cite{shalev2008svm} (also see the work of Chaudhuri  et al. in \cite{chaudhuri2011differentially}), and have the follows, 
{\small \begin{equation} 
\begin{aligned}
\hat{C}(V^*_p(t+1)) =& \hat{C}(f^0_p(t+1))
+ \big( \hat{Z}_p(V^*_p(t+1),t) -\hat{Z}_p(\hat{f}_p(t+1),t)  \big)\\
+& \big( \hat{Z}_p(\hat{f}_p(t+1),t) - \hat{Z}_p(f^0_p(t+1),t) \big)\\
+& \frac{\rho}{2}\parallel f^0_p(t+1) \parallel^2 - \frac{\rho}{2}\parallel V^*_p(t+1) \parallel^2.
\end{aligned}
\end{equation}}
If $R(f_p(t)) = \frac{1}{2}\parallel f_p(t) \parallel ^2$, then, $\parallel \nabla^2 R(f_p(t)) \parallel\leq 1$. Thus, we can apply Lemma \ref{Lemma12} with $\tau = 1$:
{\small $$
\begin{aligned}
& Z_p^{prim}(V^*_p(t+1),t|D_p)-Z_p^{prim}(f^*_p(t+1),t|D_p)\\
&\leq \frac{4 \big(C^R\big)^2 d^2\Big(\rho+ c_4C^R \Big)\Big( \ln(\frac{d}{\delta})\Big)^2}{\rho^2  B_p^2\alpha_p(t)^2},
\end{aligned}
$$  }
with probability $\geq 1-\delta$ over the noise. In the proof of Theorem 5, we have, with probability $1-\delta$
,
$$
\begin{aligned}
Z_p(f^*_p(t+1)|D_p) -Z_p(f^*|D_p)\leq \frac{4\eta^2N_p^2 d^2 \big(\ln(\frac{d}{\delta}) \big)^2}{\rho^3 B_p^2 \alpha_p(t)^2}.
\end{aligned}
$$
Therefore, with probability $1-2\delta$, we have
{\small $$
\begin{aligned}
&Z_p(V^*_p(t+1)|D_p) -Z_p(f^*|D_p)\\
&\leq \frac{4\eta^2N_p^2 d^2 \big(\ln(\frac{d}{\delta}) \big)^2}{\rho^3 B_p^2 \alpha_p(t)^2}+ \frac{4d^2\Big(\rho + c_4 \Big)\Big( \ln(\frac{d}{\delta})\Big)^2}{\rho^2  B_p^2\alpha_p(t)^2}.
\end{aligned}
$$ }

Sridharan et al. in \cite{sridharan2009fast} shows, with probability $1-\delta$,
{\small $$
\begin{aligned}
&\hat{Z}_p(V^*_p(t+1))-\hat{Z}_p(\hat{f}_p(t+1)) \\
&\leq 2\Big(Z_p^{prim}(V_p(t+1),t|D_p)-Z_p^{prim}(f^*_p(t+1),t|D_p)\Big)\\
&+ \mathcal{O}\Big( C^R\frac{\ln(\frac{d}{\delta})}{B_p\rho}\Big)\\
\leq& \frac{8 \big(C^R\big)^2 d^2\Big(\rho+ c_4C^R \Big)\Big( \ln(\frac{d}{\delta})\Big)^2}{\rho^2  B_p^2\alpha_p(t)^2}+\frac{8\eta^2N_p^2 d^2 \big(\ln(\frac{d}{\delta}) \big)^2}{\rho^3 B_p^2 \alpha_p(t)^2}\\
&+ \mathcal{O}\Big(C^R \frac{\ln(\frac{1}{\delta})}{B_p\rho}\Big).
\end{aligned}
$$ }
Combining the above two inequalities, we have the probability no smaller than $1-3\delta$.

Since $\hat{f}_p(t+1) = \arg\min_{f_p} \hat{Z}_p(f_p,t)$, then, $\big( \hat{Z}_p(\hat{f}_p(t+1),t) - \hat{Z}_p(f^0_p(t+1),t)\leq 0$. For the last two terms, we select $\rho = \frac{\alpha_{acc}}{\parallel f_p^0(t+1) \parallel^2}$ to make them bounded by $\frac{\alpha_{acc}}{2}$.

The value of $B_p$ is determined by solving 
{\small $$
\begin{aligned}
    &\frac{8 \big(C^R\big)^2 d^2\Big(\rho+ c_4C^R \Big)\Big( \ln(\frac{d}{\delta})\Big)^2}{\rho^2  B_p^2\alpha_p(t)^2}+\frac{8\eta^2N_p^2 d^2 \big(\ln(\frac{d}{\delta}) \big)^2}{\rho^3 B_p^2 \alpha_p(t)^2}   \\
&+ \mathcal{O}\Big(C^R \frac{\ln(\frac{1}{\delta})}{B_p\rho}\Big) +\frac{\alpha_{acc}}{2} = \alpha_{acc},
\end{aligned}
$$ }
with  $\rho = \frac{\alpha_{acc}}{\parallel f_p^0(t+1) \parallel^2}$, such that
$
\mathbbm{P}\big( \hat{C}(V^*_p(t+1))\leq \hat{C}^0(t+1)+\alpha_{acc} \big)\geq 1-3\delta.
$
However, the accuracy of $V^*_p(t+1) $ depends on $f^*_p(t+1)$, thus we also have to make 
$
\mathbbm{P}\big( \hat{C}(f^*_p(t+1))\leq \hat{C}^0(t+1)+\alpha_{acc} \big)\geq 1-2\delta.
$
Combining the result of Theorem 5, we arrive at (\ref{Theo_6}).

\end{proof}

%%%%%

%
\begin{lemma}\label{Lemma12}
Assume $R(f_p(t))$ is doubly differentiable w.r.t. $f_p(t)$ with $\parallel \nabla^2 R(f_p(t))\parallel\leq \tau$ for all $f_p(t)$. Suppose the loss function $\mathcal{L}$ is differentiable,  $\mathcal{L}'$ is continuous, and satisfies 
$
|\mathcal{L}'(a)-\mathcal{L}'(b)|\leq c_4|a - b|
$ for all pairs $(a,b)$ with a constant $c_4$. 
Let $f^*_p(t+1)=\arg\min_{f_p} Z_p^{prim}(f_p,t|D_p)$, 
and $V^*_p(t+1) =f^*_p(t+1)+\epsilon_p(t)$, where the noise vector $\epsilon_p(t)$ is drawn from (15) with the same $\alpha_p(t)$ for all $p\in \mathcal{P}$ at time $t$. Let $\Lambda$ be the event 
$$
\begin{aligned}
\Lambda:=& \Big\{ Z_p^{prim}(V^*_p(t+1),t|D_p)\leq Z_p^{prim}(f^*_p(t+1),t|D_p) +\psi \Big\}
\end{aligned}
$$ 
where $\psi = \frac{4 \big(C^R\big)^2 d^2\Big(\rho\tau + c_4C^R \Big)\Big( \ln(\frac{d}{\delta})\Big)^2}{\rho^2  B_p^2\alpha_p(t)^2}$.
Under Assumption 1 and 2, we have:
$
\begin{aligned}
\mathbb{P}_{\epsilon_p(t)}\Big( \Lambda \Big)\geq 1-\delta.
\end{aligned}
$
The probability $\mathbb{P}_{\epsilon_p(t)}$ is taken over the noise vector $\epsilon_p(t)$.

\end{lemma}
\begin{proof}\textbf{(Lemma \ref{Lemma12})}
From Assumption 3, we know that the data points in dataset $D_p$ satisfy: $\parallel x_{ip}\parallel\leq 1$, and $|y_{ip}|=1$. From Assumption 1 and 2, $R(\cdot)$ and $\mathcal{L}$ are differentiable. Suporse $R(\cdot)$ is doubly differentiable and $\nabla^2 R(\cdot)\leq \tau$.
Let $0\leq\varphi\leq 1$, then, the \textit{Mean Value Theorem} and Cauchy-Schwarz inequality give:
{\small $$
\begin{aligned}
&Z_p^{prim}(V^*_p(t+1),t|D_p) - Z_p^{prim}(f^*_p(t+1),t|D_p) \\
& = (V^*_p(t+1)-f^*_p(t+1))^T\nabla Z_p^{prim}\Big( \varphi f^*_p(t+1) \\
&+  (1-\varphi)V^*_p(t+1)  \Big)\leq \parallel V^*_p(t+1) - f^*_p(t+1) \parallel \\
& \cdot \parallel \nabla Z_p^{prim}\Big( \varphi f^*_p(t+1) +  (1-\varphi)V^*_p(t+1)  \Big) \parallel.
\end{aligned}
$$}
Let $\epsilon^{pi}(t)=\epsilon_p(t)-\epsilon_i(t)$. From the definition of $Z_p^{prim}(f_p,t|D_p)$, we have:
{\small $$
\begin{aligned}
&Z_p^{prim}(f_p,t|D_p) =Z_p(f_p|D_p)\\
-&\eta \sum_{i\in \mathcal{N}_p}\Big((f_p -\frac{1}{2}(f_p(t)+f_i(t))^T\cdot(\epsilon^{pi}(t))
+\frac{1}{4}(\epsilon^{pi}(t))^2 \Big).
\end{aligned}
$$}
Taking the derivative of $Z_p^{prim}$ w.r.t. $f_p$ gives
$$
\begin{aligned}
\nabla Z_p^{prim}(f_p, t|D_p) = & \frac{C^R}{B_p}\sum_{i=1}^{B_p}y_{ip}\mathcal{L}'(y_{ip}f_p^Tx_{ip})x_{ip} \\
&+ \rho \nabla R(f_p) - \eta\sum_{j\in \mathcal{N}_p}\epsilon^{pi}(t).
\end{aligned}
$$
Since $\nabla Z_p^{prim}(f^*_p(t+1), t|D_p)=0$, then, we have:
{\small $$
\begin{aligned}
&\nabla Z_p^{prim}\Big( \varphi f^*_p(t+1) +  (1-\varphi)V^*_p(t+1)  |D_p\Big)\\
& = \nabla Z_p^{prim}(f^*_p(t+1), t|D_p) - \rho\Big(\nabla R(f^*_p(t+1)) - \nabla R\big(\varphi f^*_p(t+1) \\
&+  (1-\varphi)V^*_p(t+1)\big) \Big)-\frac{C^R}{B_p}\sum_{i=1}^{B_p} \Bigg(y_{ip}\Big(   \mathcal{L}'(y_{ip}f^*_p(t+1)^Tx_{ip})\\
& -\mathcal{L}'(y_{ip}\big( \varphi f^*_p(t+1) +  (1-\varphi)V^*_p(t+1) \big)^Tx_{ip}) \Big)x_{ip}\Bigg).
\end{aligned}
$$}
$$
\begin{aligned}
\textrm{Let~}  T =& y_{ip}\Big(   \mathcal{L}'(y_{ip}f^*_p(t+1)^Tx_{ip})\\
& -\mathcal{L}'(y_{ip}\big( \varphi f^*_p(t+1) +  (1-\varphi)V^*_p(t+1) \big)^Tx_{ip}) \Big)x_{ip}.
\end{aligned}
$$
Based on the condition on the loss function:
$
|\mathcal{L}'(a)-\mathcal{L}'(b)|\leq c_4|a - b|,
$ we can bound $T$ as follows:
{\small $$
\begin{aligned}
T\leq & |y_{ip}|\parallel x_{ip} \parallel \\
&\cdot |\mathcal{L}'(y_{ip}f^*_p(t+1)^Tx_{ip})\\
& -\mathcal{L}'(y_{ip}\big( \varphi f^*_p(t+1) +  (1-\varphi)V^*_p(t+1) \big)^Tx_{ip}) |\\
\leq& |y_{ip}|\parallel x_{ip} \parallel\cdot c_4\cdot |y_{ip}(1-\varphi)(f^*_p(t+1)-V^*_p(t+1))^Tx_{ip}|\\
\leq & c_4\cdot (1-\varphi)|y_{ip}|^2\parallel x_{ip} \parallel ^2 \parallel f^*_p(t+1)-V^*_p(t+1) \parallel \\
\leq& c_4\cdot (1-\varphi) \parallel f^*_p(t+1)-V^*_p(t+1) \parallel .
\end{aligned}
$$}
Since we assume $R(\cdot)$ is doubly differentiable, we then apply the 
\textit{Mean Value Theorem}:
$$
\begin{aligned}
& \parallel \nabla R(f^*_p(t+1)) - \nabla R\big(\varphi f^*_p(t+1) +  (1-\varphi)V^*_p(t+1)\big) \parallel\\
&\leq (1-\varphi) \parallel f^*_p(t+1) -V^*_p(t+1) \parallel\cdot \parallel \nabla^2R(\xi) \parallel,
\end{aligned}
$$
where $\xi \in \mathbbm{R}^d$. Therefore, we have
{\small $$
\begin{aligned}
&\nabla Z_p^{prim}\Big( \varphi f^*_p(t+1) +  (1-\varphi)V^*_p(t+1)  |D_p\Big)\\
& \leq (1-\varphi) \parallel f^*_p(t+1) -V^*_p(t+1) \parallel\cdot \rho\cdot \parallel \nabla^2R(\xi) \parallel \\
&+ C^R c_4\cdot (1-\varphi) \parallel f^*_p(t+1)-V^*_p(t+1) \parallel \\
&\leq (1-\varphi)\cdot \parallel f^*_p(t+1)-V^*_p(t+1) \parallel\Big(\rho\tau + C^Rc_4 \Big)\\
& \leq \parallel f^*_p(t+1)-V^*_p(t+1) \parallel\Big(\rho\tau + C^Rc_4 \Big).
\end{aligned}
$$}
Since $ f^*_p(t+1)-V^*_p(t+1)=\epsilon_p(t)$, with density $\Gamma(d,\frac{2C^R}{\rho B_p\alpha_p(t)})$ then, we can apply Lemma 10 to  $\parallel f^*_p(t+1)-V^*_p(t+1) \parallel $. Thus,  with $\parallel f^*_p(t+1)-V^*_p(t+1) \parallel\leq \frac{2C^Rd\ln(\frac{d}{\delta})}{\rho B_p\alpha_p(t)}$, we have:
{\small $$
\begin{aligned}
&Z_p^{prim}(V^*_p(t+1),t|D_p) -Z_p^{prim}(f^*_p(t+1),t|D_p)\\
&\leq \frac{4 \big(C^R\big)^2 d^2\Big(\rho\tau + c_4C^R \Big)\Big( \ln(\frac{d}{\delta})\Big)^2}{\rho^2  B_p^2\alpha_p(t)^2},
\end{aligned}
$$}
with probability no less than $1-\delta$.

\end{proof}

\renewcommand\refname{Reference}
\bibliographystyle{plain} 
\bibliography{Private_Up_1-QZ_V1}

% that's all folks
\end{document}